\documentclass{article}
\usepackage{subcaption}
\usepackage{fullpage}

\usepackage[utf8]{inputenc}
\usepackage{authblk}
\usepackage{amsmath}

\title{\bf How Does Muon Work in Private Deep Learning? Scalable Differentially Private Training via Orthogonalization}

\author{Alexander Gaponov\thanks{Contacts: \texttt{agaponov04@gmail.com}} \quad Grigory Malinovsky \quad Egor Shulgin \quad Peter Richtárik}

\affil{King Abdullah University of Science and Technology (KAUST)\\ KAUST Center of Excellence for Generative AI \\ Thuwal, Saudi Arabia}

\date{}

\usepackage[round]{natbib}

\usepackage{stfloats}

\usepackage{parskip}
\usepackage{hyperref} 
\usepackage{subcaption}
\usepackage{booktabs}
\usepackage{dirtytalk}
\usepackage{makecell}
\usepackage{wrapfig}
\usepackage{tabularx}

\usepackage[hyperpageref]{backref}
\hypersetup{ 
    colorlinks=true,
    citecolor=darkscarlet,
    linkcolor=darkpowderblue,
    filecolor=magenta,      
    urlcolor=yaleblue,
    menucolor=gray
}
\renewcommand*{\backref}[1]{}
\renewcommand*{\backrefalt}[4]{%
   \ifcase #1 %
     \footnotesize{(Not cited.)}%
   \or
     \footnotesize{(Cited on page~#2)}%
   \else
     \footnotesize{(Cited on page~#2)}%
\fi }

\usepackage{multirow}
\usepackage{layout}

\usepackage{amsmath}
\usepackage{amsthm}
\usepackage{amssymb}
\usepackage{amsfonts}
\usepackage{mathtools}
\usepackage{siunitx}
\usepackage{lipsum}

\usepackage{enumitem}

\usepackage{bbm}

\usepackage{url}
\usepackage{color}
\usepackage{graphicx}
\usepackage{verbatim}
\usepackage{subcaption}

\usepackage{cleveref}

\usepackage{natbib}

\usepackage{apptools}
\usepackage{thmtools}
\usepackage{algorithm, algpseudocode}

\newcommand{\del}[1]{}

\newcommand{\clip}{\mathrm{clip}}

\usepackage{tcolorbox}
\usepackage{pifont}
\definecolor{mydarkgreen}{RGB}{39,130,67}
\definecolor{mydarkred}{RGB}{192,47,25}
\definecolor{mydarkorange}{RGB}{39,130,67}

\newtheorem{assumption}{Assumption}
\newtheorem{lemma}{Lemma}

\newtheorem{theorem}{Theorem}
\newtheorem{proposition}{Proposition}

\newtheorem{corollary}{Corollary}

\theoremstyle{plain}

\theoremstyle{definition}

\newtheorem{definition}[theorem]{Definition}

\usepackage{listings}
\usepackage{xcolor}

\definecolor{codegreen}{rgb}{0,0.6,0}
\definecolor{codegray}{rgb}{0.5,0.5,0.5}
\definecolor{codepurple}{rgb}{0.58,0,0.82}
\definecolor{backcolour}{rgb}{0.95,0.95,0.92}

\lstdefinestyle{mystyle}{
    backgroundcolor=\color{backcolour},   
    commentstyle=\color{codegreen},
    keywordstyle=\color{magenta},
    numberstyle=\tiny\color{codegray},
    stringstyle=\color{codepurple},
    basicstyle=\ttfamily\footnotesize,
    breakatwhitespace=false,         
    breaklines=true,                 
    captionpos=b,                    
    keepspaces=true,                 
    numbers=left,                    
    numbersep=5pt,                  
    showspaces=false,                
    showstringspaces=false,
    showtabs=false,                  
    tabsize=2
}

\definecolor{darkscarlet}{rgb}{0.34, 0.01, 0.1}
\definecolor{yaleblue}{rgb}{0.06, 0.3, 0.57}
\definecolor{darkpowderblue}{rgb}{0.0, 0.2, 0.6}
\definecolor{midnightblue}{rgb}{0.1, 0.1, 0.44}

\begin{document}

\maketitle

\begin{abstract}
Training deep neural networks requires optimizers that balance computational efficiency with formal data privacy.
While orthogonalized update rules, such as the recently introduced Muon optimizer, have shown strong empirical performance and scalability, their behavior under Differential Privacy (DP) is not yet understood.
In this paper, we propose DP-Muon, a differentially private variant of the Muon optimizer.
By analyzing Muon's orthogonalized updates through the lens of a Linear Minimization Oracle (LMO), we show that the leading effect of per-example clipping is a positive rescaling of the gradient, to which the LMO is blind, so that only a small directional residual controlled by the centered gradient variance enters the optimization error.
This yields an $\mathcal O(K^{-1/4})$ convergence rate, up to logarithmic factors, to a clipping-dependent neighborhood at a fixed noise scale and, with noise calibrated to a fixed $(\varepsilon,\delta)$ budget, the standard $(\sqrt d/(n\varepsilon))^{1/2}$ privacy--utility bound of DP-SGD, without assuming bounded gradients. Furthermore, we investigate the hyperparameter landscape of DP-Muon and document how the optimal learning rate, momentum, and batch size change under privacy.
Empirically, on image classification DP-Muon matches or improves on DP-SGD and DP-Adam, whereas on language modeling the best optimizer depends on the batch size, with DP-Muon best at small batches and DP-Adam best near the privacy-optimal batch size.
\end{abstract}

\section{Introduction}

Adaptive methods such as Adam \citep{kingma2014adam} and AdamW \citep{loshchilov2017decoupled} remain the default choice for training large models \citep{achiam2023gpt, liu2024deepseek}.
A prominent recent alternative is Muon \citep{jordan2024muon}, which applies an approximate orthogonalization to the momentum of matrix-valued parameters.
Muon has attracted considerable attention due to its strong empirical performance, efficient GPU implementation, and reduced optimizer-state memory; its scalability has been demonstrated by training the Moonlight 16B Mixture-of-Experts model on 5.7 trillion tokens \citep{liu2025muon}.
Its update can be interpreted through stochastic spectral descent \citep{carlson2015stochastic, kovalev2018stochastic} and, more generally, as a step along a \emph{linear minimization oracle} (LMO) over a norm ball \citep{jaggi2013revisiting, pethick2025training}.
The LMO view gives a unified treatment of non-Euclidean methods with momentum, recovering normalized SGD \citep{levy2016power, cutkosky2020momentum}, SignSGD and Signum \citep{bernstein2018signsgd, sun2023momentum}, Scion \citep{pethick2025training}, and Muon as special cases.

In this work, we study whether, and how, the advantages of momentum orthogonalization carry over to differential privacy (DP).
We consider the standard central, example-level setting of DP-SGD \citep{Abadi_2016}: each per-example gradient is clipped to norm $c$, Gaussian noise is added to the average, and the result is passed to the optimizer; the momentum and LMO steps are post-processing and do not affect privacy \citep{dwork2014algorithmic}.
The optimization problem, however, becomes substantially more delicate.
Active clipping introduces a state-dependent bias, privacy noise is injected at every step, momentum aggregates gradients along a changing trajectory, and the noisy momentum is then passed through a nonlinear map.
Existing analyses of private nonconvex optimization typically sidestep active clipping: the sharpest rates assume Lipschitz losses, that is, uniformly bounded per-example gradients, so that clipping at a threshold above the bound is inactive \citep{bassily2021non, tran2022momentum, wang2023efficient, arora2023faster, lowy2024make}, and analyses of private distributed methods make similar assumptions (Appendix~\ref{app:related}).
Generic analyses of biased stochastic gradients \citep{ajalloeian2020convergence, demidovich2023guide} measure the full additive bias and do not exploit the structure of the update.

\paragraph{Key observation: for an LMO, clipping is (almost) a rescaling.}
Let $\mu(g):=\mathbb E[\clip_c(g+\xi)]$ be the expected clipped gradient, where $g$ is the true gradient and $\xi$ is zero-mean sampling noise with variance $\sigma^2$.
We show (Lemma~\ref{lem:ray_main}) that
\begin{equation}
\label{eq:ray_intro}
\begin{gathered}
    \mu(g)=\lambda(g)\,g+r(g),\qquad
    \lambda(g)\asymp\frac{c}{\max\{c,\|g\|_2\}},\\
    \|r(g)\|_2\lesssim \lambda(g)\min\Bigl\{\sigma,\frac{\sigma^2}{c}\Bigr\}.
\end{gathered}
\end{equation}
The leading effect of clipping is thus a \emph{positive rescaling} of the true gradient, and an LMO is blind to it: $\mathrm{LMO}(\lambda g)=\mathrm{LMO}(g)$ for every $\lambda>0$.
Only the residual $r(g)$, which changes the \emph{direction}, enters the optimization error.
This lets us analyze momentum-LMO methods under active clipping assuming only bounded \emph{centered} variance, without bounded gradients and without a clipping threshold that grows with the horizon.

\paragraph{Contributions.}
Our answer to the question in the title has a theoretical and an empirical part.
\begin{itemize}[leftmargin=1.2em,itemsep=1pt,topsep=2pt]
    \item \textbf{Clipping as rescaling.} We prove the positive-ray decomposition~\eqref{eq:ray_intro} and a matching bound on the variance of clipped gradients relative to $\lambda(g)$ (Section~\ref{sec:theory}).
    \item \textbf{Convergence of private momentum-LMO methods.} For smooth nonconvex objectives in any finite-dimensional norm (covering Muon, Scion, sign descent, and normalized SGD), we prove a high-probability bound that separates optimization error, momentum tracking, sampling noise, privacy noise, and a clipping floor $\rho_1\min\{\sigma,\sigma^2/c\}$ (Theorem~\ref{thm:main}).
    The floor is the dual-norm version of the tight Euclidean floor of clipped SGD \citep{koloskova2023revisiting} and vanishes as $\sigma\to0$; passing through the LMO thus costs at most the norm-conversion constant $\rho_1$ relative to clipped SGD, and the guarantee does not show that it improves on it.
    At a fixed noise scale, the bound gives the rate $\widetilde{\mathcal O}(K^{-1/4})$, where $\widetilde{\mathcal O}$ hides factors logarithmic in $K$ and in the failure probability; and with noise calibrated to an $(\varepsilon,\delta)$ budget it gives the familiar $\mathcal O\big((c\,\mathbb E\|Z\|_\star\sqrt{\log(1/\delta)}/(n\varepsilon))^{1/2}\big)$ dependence (Corollary~\ref{cor:rates}, Appendix~\ref{app:calibrated}).
    \item \textbf{What changes under DP.} On CNN and a 30M-parameter GPT model, we sweep the batch size, learning rate, and momentum for DP-SGD, DP-Adam, and DP-Muon (and Scion).
    Under DP, the optimal batch size grows by one to two orders of magnitude and the ranking of optimizers \emph{depends on the batch size}, unlike in non-private training, where Muon is best at every batch size.
    On CIFAR-10, DP-Muon attains the lowest mean test loss at all but one batch size, with the gaps within one standard deviation at intermediate batch sizes; on language modeling, DP-Muon is best at small batch sizes, DP-Adam near the optimal batch size, and Scion at large batch sizes (Section~\ref{sec:experiments}).
\end{itemize}

\paragraph{Concurrent work.}
Two concurrent and independent works also study private Muon\footnote{Both also use the name ``DP-Muon'' for closely related algorithms.}.
\citet{kim2026dpmuon} analyze a Muon variant with finite-step Newton--Schulz (NS) orthogonalization and propose a bias-corrected DP-MuonBC; their vanishing-rate corollaries assume moment bounds on the raw (uncentered) per-example gradients and let the clipping threshold grow with the horizon, while their general bound contains a clipping term that persists even without sampling noise.
\citet{shanmugam2026spectral} study when orthogonalization recovers useful singular directions of privatized gradients through a Wedin-type signal-to-noise argument and report large gains of DP-Muon over DP-SGD in vision at large batch sizes.
Our analysis is complementary: it covers general LMO geometries, requires only centered variance, and isolates the part of the clipping bias that an LMO actually sees; a detailed comparison is given in Section~\ref{sec:comparison} and Appendix~\ref{app:comparison}.

\section{Setup and Algorithm}
\label{sec:setup}

We consider $\min_{x\in\mathcal X} f(x):=\frac1n\sum_{i=1}^n f_i(x)$ over a finite-dimensional normed space $(\mathcal X,\|\cdot\|)$ with dual norm $\|v\|_\star:=\sup_{\|d\|\le1}\langle v,d\rangle$.
For the unit ball $\mathcal C=\{v:\|v\|\le1\}$, the linear minimization oracle is $\mathrm{LMO}(g)\in\arg\min_{v\in\mathcal C}\langle g,v\rangle$, so that $\langle g,\mathrm{LMO}(g)\rangle=-\|g\|_\star$ and $\mathrm{LMO}(\lambda g)=\mathrm{LMO}(g)$ for $\lambda>0$.
The spectral norm on a weight matrix gives Muon, $\mathrm{LMO}(G)=-UV^\top$ for $G=U\Sigma V^\top$; the $\ell_\infty$ norm gives sign descent; the $\ell_2$ norm gives normalized gradient descent; and layer-wise combinations give Scion \citep{pethick2025training} (see Table~\ref{tab:lmo_norms} in the appendix).
Since $\mathcal X$ is finite-dimensional, there are constants $\rho_1,\rho_2$ with $\|v\|_\star\le\rho_1\|v\|_2$ and $\|v\|_2\le\rho_2\|v\|_\star$; for the spectral norm on an $m\times n'$ matrix, $\|\cdot\|_\star$ is the nuclear norm, $\rho_1=\sqrt{\min\{m,n'\}}$, and $\rho_2=1$.
Clipping is Euclidean, $\clip_c(v):=v\min\{1,c/\|v\|_2\}$, as in DP-SGD.

\begin{algorithm}[t]
\caption{DP-Muon (private momentum-LMO)}
\label{alg:dpmuon_}
\begin{algorithmic}[1]
    \Require stepsize $\gamma$, momentum $\alpha\in(0,1)$, clipping threshold $c$, noise scale $\sigma_{\rm DP}$, batch size $B$, iterations $K$
    \State $m_0 \gets 0$
    \For{$k= 0, \dots, K-1$}
        \State Sample $S_k\subseteq[n]$, $|S_k|=B$, uniformly at random
        \State $\hat{g}_k \gets \frac{1}{B} \sum_{i \in S_k} \clip_c (\nabla f_{i}(x_k)) + z_k$, \ $z_k \sim \mathcal{N}(0, c^2\sigma_{\rm DP}^2 I)$
        \State $m_{k+1} \gets (1-\alpha) m_k + \alpha \hat{g}_k$
        \State $x_{k+1} \gets x_k + \gamma \operatorname{LMO}(m_{k+1})$
    \EndFor
\end{algorithmic}
\end{algorithm}

DP-Muon (Algorithm~\ref{alg:dpmuon_}) privatizes the mini-batch gradient exactly as DP-SGD does; only the optimizer step differs.
In practice, Muon is applied with a few Newton--Schulz iterations in place of the exact polar factor; our analysis assumes an exact LMO.

A randomized algorithm $\mathcal M$ is $(\varepsilon,\delta)$-differentially private if $\mathbb P(\mathcal M(\mathcal D)\in S)\le e^{\varepsilon}\,\mathbb P(\mathcal M(\mathcal D')\in S)+\delta$ for all measurable $S$ and all adjacent datasets $\mathcal D,\mathcal D'$ \citep{dwork2014algorithmic}.
We use add/remove adjacency (one example added or removed) for Poisson subsampling, as in \citet{Abadi_2016}, and replace-one adjacency for fixed-size sampling without replacement \citep{wang2019subsampled}; an $(\varepsilon,\delta)$-DP guarantee under add/remove adjacency implies $(2\varepsilon,(1+e^{\varepsilon})\delta)$-DP under replace-one adjacency, but not conversely.

\begin{proposition}[Privacy]
\label{prop:privacy}
Let $x_0$, $K$, $B$, $c$, $\sigma_{\rm DP}$, $\alpha$, and $\gamma$ be fixed independently of the data (and, under replace-one adjacency, $n$ be public), and let the output be the whole trajectory $x_0,\dots,x_K$, so that any function of it, such as the tail average or a uniformly sampled iterate, inherits the guarantee.
The privacy of Algorithm~\ref{alg:dpmuon_} is that of DP-SGD with the same sampling, clipping, and noise: the momentum and LMO steps only post-process the privatized gradients $\hat g_0,\dots,\hat g_{K-1}$.
In particular, with $q=B/n$, there are constants $c_1,c_2>0$ such that Algorithm~\ref{alg:dpmuon_}, with fixed-size sampling without replacement, is $(\varepsilon,\delta)$-DP under replace-one adjacency for any $\varepsilon<c_1q^2K$ if
\begin{equation}
\label{eq:dp_noise}
    \sigma_{\rm DP}\ \ge\ \frac{c_2\sqrt{K\log(1/\delta)}}{n\varepsilon}.
\end{equation}
If the fixed-size sampling is replaced by Poisson subsampling at rate $q$ (normalized by the expected batch size $B$), a calibration of the same form, with different constants, holds under add/remove adjacency \citep[Theorem~1]{Abadi_2016}.
\end{proposition}

\begin{proof}[Proof sketch]
The first claim is the post-processing property of differential privacy \citep{dwork2014algorithmic}, applied with adaptive composition over the $K$ releases $\hat g_k$.
For fixed-size sampling, replacing one example changes the clipped sum $\sum_{i\in S_k}\clip_c(\nabla f_i(x_k))$ by at most $2c$ in $\ell_2$, so each release is a Gaussian mechanism with noise multiplier $\sigma_{\rm mult}=B\sigma_{\rm DP}$ relative to sensitivity $c$.
For the subsampled Gaussian mechanism, \citet[Theorem~9 and the discussion after it]{wang2019subsampled} give the R\'enyi-DP bound $\varepsilon'(a)=\mathcal O(q^2a/\sigma_{\rm mult}^2)$ at small orders $a$ and large $\sigma_{\rm mult}$; the replace-one sensitivity $2c$ only changes the constant.
Composing over $K$ steps and converting to $(\varepsilon,\delta)$ with an order $a\asymp\log(1/\delta)/\varepsilon$, exactly as in the proof of \citet[Theorem~1]{Abadi_2016}, gives~\eqref{eq:dp_noise}; the requirement that this order lies in the range where the bound holds is what imposes $\varepsilon<c_1q^2K$.
\end{proof}

Note that, for a fixed number of steps $K$, the noise scale on the \emph{averaged} gradient, $c\,\sigma_{\rm DP}$, does not depend on $B$; at a fixed budget of epochs, $K\propto1/B$ and the noise on the average decreases with $B$ only through the smaller number of steps (Appendix~\ref{app:scaling_accountant} quantifies this with the accountant actually used in the experiments).

\begin{assumption}
\label{assum:smoothness}
$f$ is bounded below by $f^\star$, and $\nabla f$ is $L$-Lipschitz: $\|\nabla f(x)-\nabla f(y)\|_\star\le L\|x-y\|$ for all $x,y$.
\end{assumption}

\begin{assumption}
\label{assum:stochastic_oracle}
For all $x$, $\frac1n\sum_{i=1}^n\|\nabla f_i(x)-\nabla f(x)\|_2^2\le\sigma^2$.
\end{assumption}

Assumption~\ref{assum:stochastic_oracle} bounds only the \emph{centered} spread of per-example gradients; in particular, the gradients themselves may be arbitrarily large, so clipping can be active throughout training.

\section{Theory: How Muon Handles Clipping}
\label{sec:theory}

\subsection{Clipping is a rescaling up to a small residual}

For a fixed point $x$, write $g=\nabla f(x)$ and $\xi_i=\nabla f_i(x)-g$, so that $\mu(g)=\frac1n\sum_i\clip_c(g+\xi_i)$ is the expected clipped mini-batch gradient, and let $a_c(g):=c/\max\{c,\|g\|_2\}$ denote the factor by which $g$ itself would be clipped.

\begin{lemma}[Positive-ray decomposition]
\label{lem:ray_main}
Under Assumption~\ref{assum:stochastic_oracle}, for every $g$ there is $\lambda(g)\in[\tfrac14a_c(g),\,2a_c(g)]$ such that
\[
    \|\mu(g)-\lambda(g)g\|_2\le16\,\lambda(g)\min\Bigl\{\sigma,\frac{\sigma^2}{c}\Bigr\},
\]
and $\frac1n\sum_i\|\clip_c(g+\xi_i)-\mu(g)\|_2^2\le320\,\lambda(g)^2\sigma^2$.
\end{lemma}

The scale $\lambda(g)$ is comparable to the clipping factor of the \emph{true} gradient, so the residual is small \emph{relative to the shrunken signal} $\lambda(g)g$: the clipped mean points in the right direction up to an angle controlled by $\min\{\sigma,\sigma^2/c\}/\|g\|_2$.
When $c\ge\sigma$ and $\|g\|_2\ge4\sigma^2/c$, the proof takes $\lambda(g)=\langle\mu(g),g\rangle/\|g\|_2^2$, so the residual is the component orthogonal to $g$ of $\mathbb E[a_c(g+\xi)\xi]=\mathbb E[(a_c(g+\xi)-a_c(g))\xi]$, which is second order in $\xi$ by the two-point estimate $|a_c(y)-a_c(g)|\le a_c(y)a_c(g)\|y-g\|_2/c$. When $c\ge\sigma$ and $\|g\|_2<4\sigma^2/c$, it takes $\lambda(g)=a_c(g)$, and the residual contains in addition the longitudinal term $(\mathbb E\,a_c(g+\xi)-a_c(g))g$, which the same estimate bounds by $4a_c(g)\sigma^2/c$ because $g$ is small. When $c<\sigma$ and $\|g\|_2\gg\sigma$, most per-example gradients are clipped onto a cone around $g$.
Both cases use only the centered variance (proofs in Appendix~\ref{app:projective_lmo_lemma}).
\citet{koloskova2023revisiting} bound the additive deviation $\|\mu(g)-\clip_c(g)\|_2$ from the same ray, since $\clip_c(g)=a_c(g)g$; the new element here is that the deviation is bounded \emph{relative to} $\lambda(g)$, uniformly in $\|g\|_2$, which is exactly what a scale-invariant update needs.
Relatedly, \citet{xiao2023theory} identify the sampling noise of the gradients being clipped as a key source of the bias of clipped and normalized DP-SGD.

The second ingredient translates a small \emph{projective} error into a descent direction.

\begin{lemma}[Projective LMO bound]
\label{lem:projective_main}
For all $m,g\in\mathcal X^\star$ and $d\in\mathrm{LMO}(m)$,
\[
    \langle g,d\rangle\ \le\ -\|g\|_\star+2\,\mathcal R(m,g),
\]
where $\mathcal R(m,g):=\inf_{\lambda>0}\|m-\lambda g\|_\star/\lambda$.
\end{lemma}

This is the normalized-step inequality of \citet[Lemma~2]{cutkosky2020momentum} (see also \citealp{pethick2025training}) applied to $m/\lambda$, using $\mathrm{LMO}(m)=\mathrm{LMO}(m/\lambda)$ and optimizing over $\lambda$.
Unlike the additive bias $\|m-g\|_\star$, $\mathcal R(m,g)$ vanishes whenever $m$ is \emph{any} positive multiple of $g$.

\subsection{Convergence of DP-Muon}

Let $k_{\rm b}:=\lceil\log2/\alpha\rceil$, $T:=K-k_{\rm b}$, $\Phi_0:=f(x_0)-f^\star$, $\bar\Phi:=(\sqrt{\Phi_0}+k_{\rm b}\gamma\sqrt{L/2})^2$, $\widetilde\sigma:=\rho_1\sigma$, $\widetilde b_c:=\rho_1\min\{\sigma,\sigma^2/c\}$, and, for $\eta\in(0,1)$ and $Z\sim\mathcal N(0,I)$,
\[
    \nu_{\rm DP}(\eta):=c\,\sigma_{\rm DP}\sqrt{3\alpha}\Bigl(\mathbb E\|Z\|_\star+\rho_1\sqrt{2\log(T/\eta)}\Bigr),
\]
the high-probability size of the momentum-filtered privacy noise.

\begin{theorem}
\label{thm:main}
Let Assumptions~\ref{assum:smoothness}--\ref{assum:stochastic_oracle} hold, $T>0$, and
\begin{equation}
\label{eq:main_conditions}
    \frac{\rho_2L\gamma}{c}\le\frac{\alpha}{4},
    \qquad
    \frac{10\rho_2\,\nu_{\rm DP}(\eta)}{c}\le\frac12.
\end{equation}
Then, with probability at least $1-\eta$ over the privacy noise, the iterates of Algorithm~\ref{alg:dpmuon_} satisfy
\[
\begin{aligned}
    &\frac1T\sum_{k=k_{\rm b}}^{K-1}\mathbb E_S\|\nabla f(x_k)\|_\star
    \le
    \frac{2\bar\Phi}{\gamma T}+\frac{45L\gamma}{\alpha}+64\,\widetilde b_c
    \\
    &\qquad+C\frac{\widetilde\sigma}{\sqrt B}\Bigl(\sqrt\alpha+\frac1{\alpha T}\Bigr)+20\,\nu_{\rm DP}(\eta),
\end{aligned}
\]
where $\mathbb E_S$ is the expectation over mini-batch sampling given the privacy noise and $C$ is a universal constant.
\end{theorem}

The first condition in~\eqref{eq:main_conditions} ties the stepsize to the clipping threshold and the momentum: the iterates must move slowly enough that the clipping scales $\lambda(\nabla f(x_k))$ seen by the momentum buffer change by at most a constant factor over its effective window $1/\alpha$.
The second requires the filtered privacy noise to be smaller than the clipping radius, which is the norm of the clipped signal when gradients are large.
Because the noise is controlled through its worst-case dual norm $\mathbb E\|Z\|_\star$, this condition is restrictive: in the experiments of Section~\ref{sec:experiments}, even in the optimistic $\ell_2$ specialization it holds, together with the run-in requirement $T>0$, only for CIFAR-10 at the largest batch size with momentum $0.99$, for no language-modeling configuration, and not at the observed optima; in the spectral geometry of DP-Muon, where $H^2/d$ grows with the matrix rank, it holds for none of our configurations (Appendix~\ref{app:scaling_batch} gives the numbers).

\begin{corollary}
\label{cor:rates}
For fixed $c$ and $\sigma_{\rm DP}$, the choice $\alpha=\Theta(1/\sqrt{K\log K})$ and $\gamma=\Theta(K^{-3/4}(\log K)^{-1/4})$ satisfies~\eqref{eq:main_conditions} for $K$ large enough (of order $\rho_2^4\sigma_{\rm DP}^4H^4$ up to logarithms, Appendix~\ref{sec:simplified_analysis}) and gives
\[
    \frac1T\sum_{k=k_{\rm b}}^{K-1}\mathbb E_S\|\nabla f(x_k)\|_\star
    =\mathcal O\Bigl(\frac{(\log K)^{1/4}}{K^{1/4}}\Bigr)+\mathcal O(\widetilde b_c).
\]
With $\sigma_{\rm DP}$ calibrated by~\eqref{eq:dp_noise} to a fixed $(\varepsilon,\delta)$ budget for the fixed-size sampling of Algorithm~\ref{alg:dpmuon_}, a suitable choice of $K,\alpha,\gamma$ gives instead
\[
    \mathcal O\Bigl((L\Phi_0)^{1/4}\Bigl(\frac{c\,H\sqrt{\log(1/\delta)}}{n\varepsilon}\Bigr)^{1/2}\Bigr)+\mathcal O(\widetilde b_c),
\]
with $H=\mathbb E\|Z\|_\star+\rho_1\sqrt{2\log(K/\eta)}$, whenever $n\varepsilon$ is large enough; in particular, it suffices that $n\varepsilon\gtrsim H\sqrt{\log(1/\delta)}\max\{\rho_2^2\sqrt{L\Phi_0}/c,\ c/\sqrt{L\Phi_0}\}$ up to lower-order terms (Appendix~\ref{app:calibrated}).
\end{corollary}

The fixed-noise rate matches the $K^{-1/4}$ rate of normalized SGD with momentum \citep{cutkosky2020momentum} and of clipped SGD \citep{koloskova2023revisiting} (the first-moment counterpart of the $K^{-1/2}$ squared-norm rate: a squared-norm rate implies ours by Jensen's inequality, not conversely), which is optimal under bounded variance \citep{arjevani2023lower}, up to the $(\log K)^{1/4}$ factor from the high-probability bound over the privacy noise.
For the $\ell_2$ geometry, $H\approx\sqrt d$ and the calibrated bound is the usual nonconvex DP-SGD dependence $(d\log(1/\delta))^{1/4}/\sqrt{n\varepsilon}$. The privacy term of the bound is below the initial gradient bound $\sqrt{2L\Phi_0}$ when $n\varepsilon\gtrsim cH\sqrt{\log(1/\delta)}/\sqrt{L\Phi_0}$, but the whole bound is nontrivial only if, in addition, the clipping floor $\widetilde b_c$, which does not depend on $n\varepsilon$, is small compared with $\sqrt{L\Phi_0}$. The condition on $n\varepsilon$ in Corollary~\ref{cor:rates} is stronger than this threshold by the factor $\rho_2^2L\Phi_0/c^2$ whenever $c\lesssim\sqrt{L\Phi_0}$, that is, the requirement on $n\varepsilon$ grows with the ratio of the gradient scale to the clipping threshold, which is the price of not assuming bounded gradients (Appendix~\ref{app:calibrated}). For the spectral geometry, $\mathbb E\|Z\|_\star$ is the nuclear norm of a Gaussian matrix, of order $\min\{m,n'\}\sqrt{\max\{m,n'\}}$ (Table~\ref{tab:lmo_norms}), and the bound controls the nuclear norm of the gradient. Translated to Frobenius stationarity, the ratio of the leading privacy terms of the spectral and Euclidean bounds is $(L_{\rm op}/L_F)^{1/4}\,r^{1/4}/s_0$, where $r=\min\{m,n'\}$, $L_{\rm op}$ and $L_F$ are the smoothness constants of Assumption~\ref{assum:smoothness} in the two geometries, and $s_0\in[1,\sqrt r]$ is such that $\|\nabla f\|_{\rm nuc}\ge s_0\|\nabla f\|_F$ along the trajectory. For gradients with a flat spectrum ($s_0=\sqrt r$) and comparable smoothness constants this is a gain of $r^{-1/4}$, for rank-one gradients a loss of $r^{1/4}$, and the gain disappears when $L_{\rm op}/L_F$ grows like $r$, as it does for $f=\tfrac12\|X\|_F^2$; the advantage of the spectral geometry therefore rests on $L_{\rm op}\ll rL_F$, which is the premise of Muon itself. This compares only the leading privacy terms: after the same conversion, the clipping floor $\widetilde b_c=\rho_1\min\{\sigma,\sigma^2/c\}$, with $\rho_1=\sqrt r$ in the spectral geometry, is $\sqrt r/s_0$ times its Euclidean counterpart, so the two floors coincide for flat spectra and the spectral one is larger by $\sqrt r$ for rank-one gradients. The worst-case dual-norm treatment of the noise is the price of the generality in the norm.

\paragraph{Proof sketch.}
The LMO is invariant to the bias-correction factor of the momentum, so we work with $\overline m_k=\sum_{t\le k}w_{k,t}\hat g_t$, $\sum_tw_{k,t}=1$.
By Lemma~\ref{lem:ray_main}, $\overline m_k=A_k\nabla f(x_k)+D_k+\sum_tw_{k,t}r_t+N_k+P_k$ with $A_k=\sum_tw_{k,t}\lambda_t$, a drift term $D_k$, sampling noise $N_k$, and privacy noise $P_k$.
Lemma~\ref{lem:projective_main} with $\lambda=A_k$ reduces the descent step to bounding each term divided by $A_k$.
The residuals contribute $16\widetilde b_c$.
The first condition in~\eqref{eq:main_conditions} keeps $A_k$ within a constant factor of $a_c(\nabla f(x_k))$, which controls the drift ($\lesssim L\gamma/\alpha$) and, with the variance bound of Lemma~\ref{lem:ray_main}, the sampling noise.
Gaussian concentration bounds $\|P_k\|_\star\le\nu_{\rm DP}(\eta)$ for all $k$, and dividing by $A_k$ produces the $\|\nabla f(x_k)\|_\star$-proportional term absorbed by the second condition.
Full proofs are in Appendix~\ref{app:convergence}.

\subsection{Comparison with existing guarantees}
\label{sec:comparison}

Table~\ref{tab:comparison} compares guarantees for mechanisms that clip \emph{individual} examples before aggregation, as DP-SGD and Algorithm~\ref{alg:dpmuon_} do.
Two points deserve emphasis.
First, our clipping floor $\rho_1\min\{\sigma,\sigma^2/c\}$ is the dual-norm version of the floor of \citet{koloskova2023revisiting}, which is tight for clipped SGD under bounded variance; passing through a non-Euclidean LMO thus costs only the norm-conversion constant $\rho_1$.
Second, the floor depends only on the \emph{centered} variance and vanishes as $\sigma\to0$, even when $\|\nabla f\|_2\gg c$.
In contrast, the general bound of \citet{kim2026dpmuon} contains a clipping term of the form $\mathbb E(\|\nabla f_i\|-c)_+$, which persists for deterministic gradients larger than $c$, and their vanishing-rate corollaries require moment bounds on raw gradients, which by $\frac1n\sum_i\|\nabla f_i\|_2^2=\|\nabla f\|_2^2+\frac1n\sum_i\|\nabla f_i-\nabla f\|_2^2$ also bound the mean gradient, together with a threshold growing with $K$. Outside privacy, \citet{cutkosky2021high} analyze the same clip--momentum--normalize pipeline in general norms under heavy-tailed noise, with an additive clipping bias controlled by a threshold that grows with the horizon.
Rates are stated under different conventions (fixed noise versus calibrated noise, average versus minimum, norm versus squared norm), which we align in Appendix~\ref{app:comparison}.

\begin{table*}[t]
\centering
\footnotesize
\setlength{\tabcolsep}{4pt}
\renewcommand{\arraystretch}{1.15}
\newcolumntype{L}[1]{>{\leavevmode\raggedright\arraybackslash}p{#1}}
\newcolumntype{Y}{>{\leavevmode\raggedright\arraybackslash}X}
\begin{tabularx}{\textwidth}{@{} L{2.6cm} L{3.2cm} L{2.3cm} Y L{2.0cm} @{}}
\toprule
\textbf{Work} & \textbf{Gradient assumption} & \textbf{Noise in rate} & \textbf{Guarantee} & \textbf{Momentum--LMO} \\
\midrule
\citet{koloskova2023revisiting}
& centered variance
& fixed
& avg.\ $\mathbb E\|\nabla f\|\lesssim K^{-1/4}+\min\{\sigma,\sigma^2/c\}$
& -- \\
\citet{yang2022normalized}
& a.s.\ relative envelope
& fixed / calibrated
& $\mathbb E\min_k\|\nabla f\|\lesssim K^{-1/4}$; calibrated: $(dc^3\log\frac1\delta)^{1/4}(n\varepsilon)^{-1/2}$
& -- \\
\citet{chen2020understanding}
& none (distance to symmetric)
& calibrated
& weighted clipped stationarity $\lesssim c\sqrt d/(n\varepsilon)$ $+$ Wasserstein term
& -- \\
\citet{shulgin2024convergence}
& centered $q$-th moment
& fixed
& weighted avg.\ $\mathbb E\|\nabla f\|^2\lesssim K^{-1/2}+\mathcal N_{\rm QC}$
& -- \\
\citet{vatankhah2025differentially}
& centered $p$-th moment
& fixed / calibrated
& $\min_k\|\nabla f\|^2\lesssim\widetilde{\mathcal O}(K^{-1/2})+\mathcal N_{\rm clip}$ (w.h.p.)
& -- \\
\citet{kim2026dpmuon}
& raw-gradient moment
& fixed multiplier, $c$ grows with $K$
& avg.\ $\mathbb E\|\nabla f\|_{\rm nuc}\lesssim K^{-1/4}\sqrt{\log K}$ or $K^{-1/6}$
& matrix, finite NS \\
\textbf{This work}
& \textbf{centered variance}
& fixed / calibrated
& $\mathbb E_S\|\nabla f\|_\star\lesssim(\log K)^{1/4}K^{-1/4}+\widetilde b_c$ (tail avg., w.h.p.)
& any norm, exact LMO \\
\bottomrule
\end{tabularx}
\caption{Optimization guarantees for per-example clipping. ``Fixed'': the rate in $K$ holds at a constant noise scale, so the privacy budget grows with $K$; ``calibrated'': noise calibrated to a fixed $(\varepsilon,\delta)$ is substituted into the bound. Constants and logarithms are suppressed unless shown. Theorem numbers, exact assumptions, and conventions are listed in Appendix~\ref{app:comparison}.}
\label{tab:comparison}
\end{table*}

\section{Experiments}
\label{sec:experiments}

\paragraph{Setup.}
We compare DP-SGD, DP-Adam, and DP-Muon (along with DP-Scion for language modeling) against their non-private counterparts on CIFAR-10 image classification~\citep{krizhevsky2009learning} and on language modeling using a 30M-parameter NanoGPT model. To approximate the orthogonalization update, we utilize the Newton-Schulz method~\cite{jordan2024muon} with 5 iterations.
DP training uses Opacus \citep{yousefpour2021opacus} with Poisson subsampling (our analysis uses fixed-size batches), clipping threshold $c=1$ following \citet{de2022unlocking}, $\varepsilon=8$, and $\delta=1/(2n)$.
For each optimizer and batch size $B$ we grid-search the learning rate and the momentum $1-\alpha$, with zero weight decay; grids are listed in Appendix~\ref{app:experiments}.
To make models compatible with per-example clipping, BatchNorm layers are replaced by GroupNorm and biases are disabled in all linear layers. We report the best value over the grid for each $B$.
The CIFAR-10 curves average three seeds; all language-modeling results and all heatmaps are single runs. For DP-SGD on CIFAR-10 ($B\ge256$) and for Scion at $B=16{,}384$, the best learning rate lies at the upper end of the grid (Appendix~\ref{app:experiments}), so these baselines may be slightly under-tuned.
Note that the hyperparameter search itself is not included in the privacy accounting, as is common in the DP literature \citep{papernot2022hyperparameter}. Note that we utilize standard momentum for SGD and Muon optimizers (equivalent to \texttt{nesterov=False} in PyTorch).

\begin{figure*}[]
    \centering
    \includegraphics[width=0.48\textwidth]{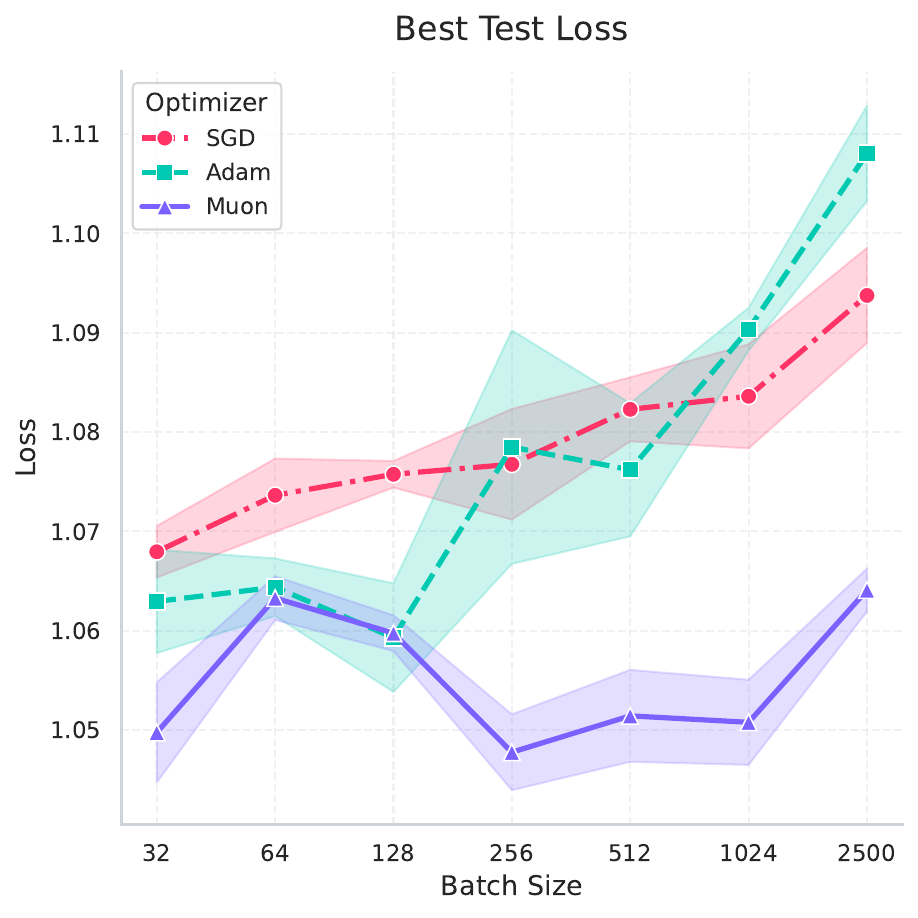}
    \hfill
    \includegraphics[width=0.48\textwidth]{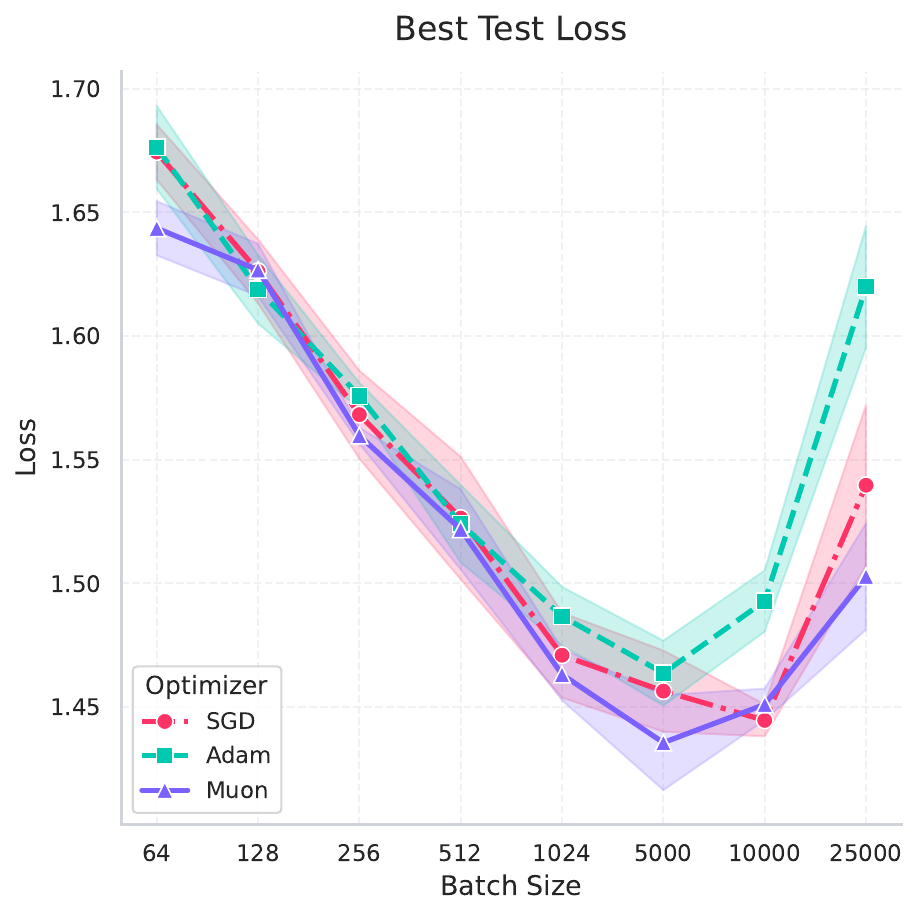}
    \caption{Best test loss as a function of batch size for SGD, Adam, and Muon. All plotted values represent the optimal configuration found across the learning rate and momentum grid search. \textbf{(Left)} Standard non-private training. \textbf{(Right)} Differentially private (DP) training for 80 epochs. We perform three independent runs to compute the reported deviation intervals.}
    \label{fig:cnn_baseline_loss_vs_bs}
\end{figure*}

\paragraph{Image classification.}
We train the CIFAR Airbench architecture \citep{jordan_cifar10_airbench} (2M parameters). Without privacy (Figure~\ref{fig:cnn_baseline_loss_vs_bs}, left), Muon attains the lowest test loss at every batch size, with its minimum (below $1.05$) at $B=256$, whereas Adam and SGD are best at $B\le128$. With privacy (right), the picture changes in two ways. First, all optimizers prefer batches that are one to two orders of magnitude larger, with minima at $B=5000$--$10000$, in line with prior observations for DP-SGD \citep{de2022unlocking, sander2023tanburnscalinglaws}. Second, the gaps between optimizers largely disappear: for $128\le B\le1024$ the three curves overlap within one standard deviation. DP-Muon attains the lowest mean loss at every batch size except $B=10000$, where DP-SGD is marginally better, but the differences exceed one standard deviation only at the extremes $B=64$ and $B=25000$, where DP-Muon degrades the least. Across privacy levels at $B=1024$ (Appendix~\ref{app:epsilon_ablation}), DP-Muon is on par with or slightly better than DP-SGD and DP-Adam for $\varepsilon\in \{1,5,8\}$, with the largest margin over DP-Adam at $\varepsilon=1$; at $\varepsilon=3$ the three methods are within noise.

\begin{figure*}[t]
    \centering
    \includegraphics[width=0.48\textwidth]{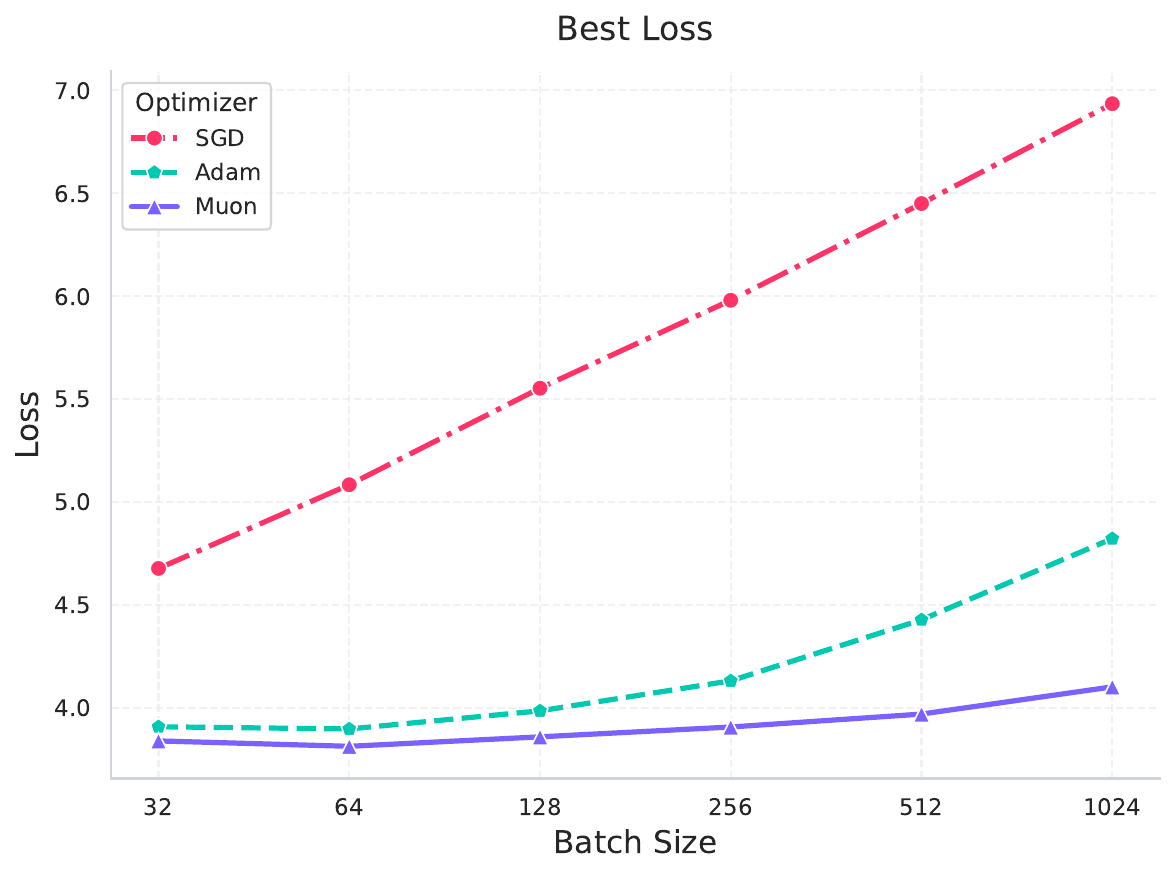}
    \hfill
    \includegraphics[width=0.48\textwidth]{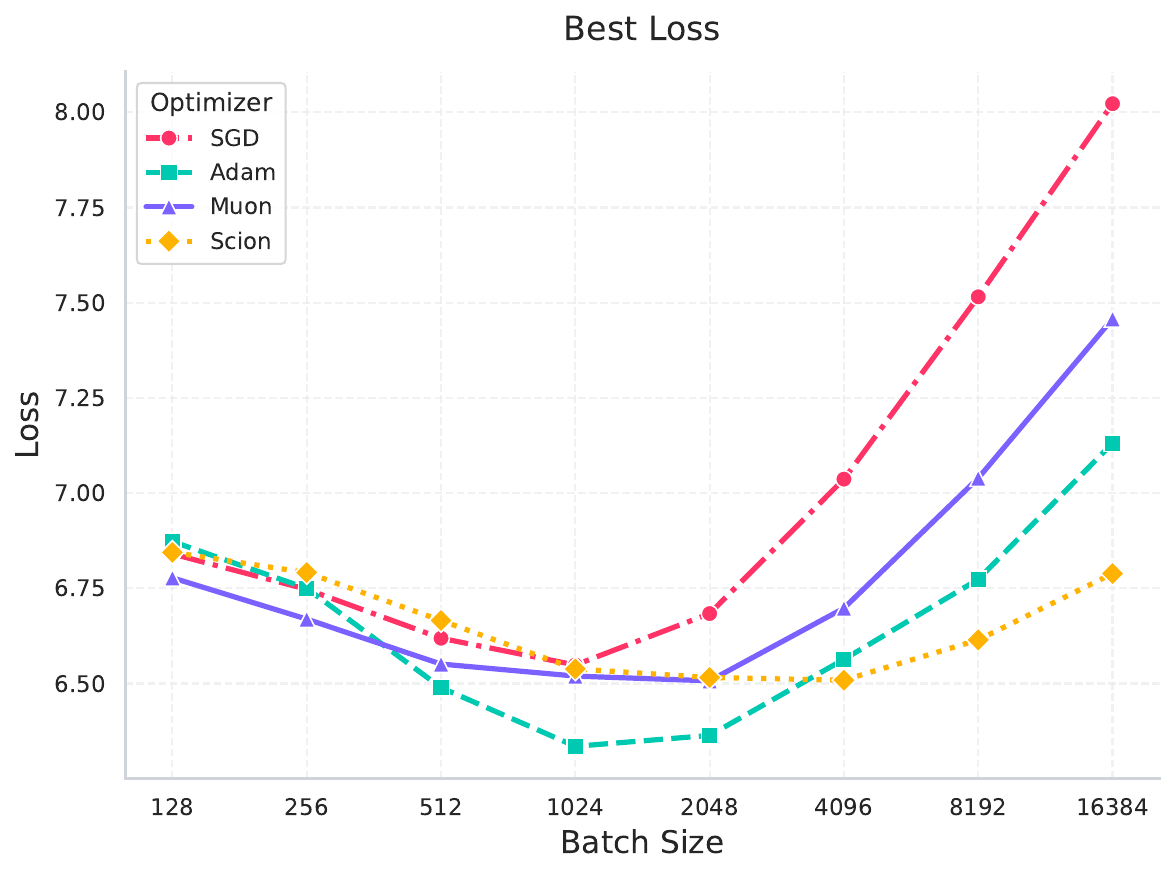}
    \caption{Language modeling (30M GPT): best loss versus batch size at the best learning rate and momentum for each batch size. \textbf{Left:} non-private. \textbf{Right:} DP ($\varepsilon=8$), including Scion. Note the different batch-size and loss ranges.}
    \label{fig:babygpt_baseline_loss_vs_bs}
\end{figure*}

\paragraph{Language modeling.}
The model consists of 6 layers, 6 attention heads, an embedding dimension of 384, and a vocabulary of 50{,}304 tokens. It is trained on sequences of length 1{,}024 for 600 million tokens from the 10-billion-token subset of FineWeb~\citep{penedo2024finewebdatasetsdecantingweb}, equating to roughly 20 tokens per parameter.

The training set is a fixed random subset of $n\approx5.86\cdot10^{5}$ sequences from this pool, processed for one epoch with Poisson sampling rate $q=B/n$ and $\delta=1/(2n)$; this translates to approximately 18{,}300 steps for a batch size of $B=32$, down to 36 steps for a maximum batch size of $B=16{,}384$. We define a single sequence of 1{,}024 tokens as the privacy unit, so a document that spans several sequences is not protected as a whole.

In the non-private baseline, the performance ranking (Muon $<$ Adam $<$ SGD) holds across all batch sizes, and Muon's advantage grows with $B$ (achieving a final loss of 4.10 versus 4.82 for Adam at $B=1024$). Under differential privacy (right), the optimal batch size shifts from 32--64 up to 1024--4096, and the optimizer ranking becomes dependent on $B$. DP-Muon has the lowest loss for $B \le 256$, by at most $0.09$ nats in single runs; DP-Adam is optimal for $512 \le B \le 2048$ and attains the lowest overall loss (6.33 at $B=1024$, versus 6.51 for DP-Muon at $B=2048$); and DP-Scion, which utilizes sign updates for the embedding and output layers, performs best for $B \ge 4096$ and exhibits the least degradation at very large batch sizes. At their respective optimal batch sizes, the four optimizers are within 0.25 nats of each other, compared to a gap of 0.9 nats between Muon and SGD without privacy.

\paragraph{What changes under privacy.}
We highlight three findings.
\emph{(i) Batch size is the dominant hyperparameter} on both tasks: at a fixed budget of epochs or tokens, the optimal batch size grows by a factor of 15--300, and the ranking of optimizers depends on it, so conclusions drawn at a single batch size, or transferred from non-private training, can be misleading. In both tasks the root-mean-square size of the privacy noise exceeds the root-mean-square bound $c/\sqrt B$ on the sampling noise of clipped gradients by one to two orders of magnitude at every batch size and decays with $B$ faster than $B^{-1/2}$, and the observed optima lie where its norm becomes comparable to the clipping radius (Figure~\ref{fig:dp_noise_scaling}, Appendix~\ref{app:scaling_accountant}). There, the noise-to-signal ratio is about $4$--$27$ times larger than the admissibility condition of Theorem~\ref{thm:main} allows for $1-\alpha=0.9$, so the theorem does not cover these configurations and does not explain the location of the optima (Appendix~\ref{app:scaling_batch}).
The dependence of the ranking on the batch size is consistent with reports of large gains for private Muon at large batches in vision \citep{shanmugam2026spectral}.
\emph{(ii) The learning rate of DP-Muon grows with the batch size}: for language modeling, from $0.05$ at $B\le512$ to $0.2$ at $B\ge8192$, whereas the optimal learning rate of DP-Adam stays within $[10^{-3},3\cdot10^{-3}]$ and, without privacy, the optimal learning rate of every method on language modeling is the same at all batch sizes (Appendix~\ref{app:llm_details}).
\emph{(iii) DP-Muon is insensitive to the momentum}: for DP language modeling, at the optimal learning rate, its loss varies by about $0.1$ nats or less over momentum values $1-\alpha\in[0.8,0.95]$ at every batch size, so the selected momentum appears erratic across batch sizes (Appendix~\ref{app:llm_details}). This is compatible with Theorem~\ref{thm:main}, in which a longer momentum window $1/\alpha$ reduces the filtered sampling and privacy noise ($\propto\sqrt\alpha$) but increases the tracking error ($\propto L\gamma/\alpha$).

\section{Related Work}
\label{sec:related-work}

\paragraph{LMO-based methods.}
Beyond Scion and the trust-region analyses of \citet{kovalev2025understanding, kovalev2025non}, LMO-based methods include the layer-wise Gluon \citep{riabinin2025gluon}, and the effect of replacing the polar factor by a few Newton--Schulz steps is analyzed by \citet{shulgin2025beyond} and \citet{kim2026muonns}.
None of these works consider per-example clipping or privacy noise; distributed and compressed variants are discussed in Appendix~\ref{app:related}.

\paragraph{Clipping in private optimization.}
Analyses of per-example clipping are compared in Section~\ref{sec:comparison}.
A complementary line replaces clipping by per-sample normalization, which removes the threshold \citep{bu2023automatic, yang2022normalized}; we instead keep the standard clipping of DP-SGD and normalize only the aggregated momentum through the LMO.
Nonlinear post-processing of privatized gradients creates its own bias, which can be corrected for Adam \citep{tang2024dpadambc} and for Newton--Schulz \citep{kim2026dpmuon}.

\paragraph{Training dynamics and hyperparameters under privacy.}
Large batches are crucial for accurate private training \citep{de2022unlocking}; \citet{sander2023tanburnscalinglaws} show that the performance of DP-SGD is governed mainly by the total amount of injected noise, which allows hyperparameters to be tuned at small batch sizes, \citet{panda2024new} propose a linear scaling rule to transfer hyperparameters across privacy budgets, \citet{mckenna2025scaling} derive scaling laws for private language models, and \citet{papernot2022hyperparameter} account for the privacy cost of tuning.
Under privacy, the noise dominates Adam's second-moment estimate, so that DP-Adam behaves like DP-SGD unless corrected \citep{tang2024dpadambc}, and \citet{compagnoni2026adaptive} compare DP-SignSGD with DP-SGD through stochastic differential equations.
For non-private LMO methods, \citet{shulgin2026deriving} derive hyperparameter scaling laws by minimizing convergence bounds of the same form as Theorem~\ref{thm:main}.

\section{Conclusion and Limitations}

Muon is compatible with differential privacy because, for a scale-invariant update, the leading effect of per-example clipping is a harmless rescaling; the remaining directional bias is controlled by the centered variance alone, which yields convergence guarantees for private momentum-LMO methods in any norm whose clipping floor matches that of clipped SGD up to the norm-conversion constant $\rho_1$. Empirically, privacy makes the batch size the dominant hyperparameter and reorders optimizers, so DP-Muon is most beneficial in some batch-size regimes rather than uniformly.
Our analysis has limitations: it assumes an exact LMO rather than a few Newton--Schulz steps and fixed-size batches, whereas our experiments use Poisson subsampling; it treats the Gaussian noise through worst-case dual-norm bounds, which leads to the admissibility condition in~\eqref{eq:main_conditions} and excludes the hyperparameter regime of our experiments (entirely in the spectral geometry of DP-Muon, and except for CIFAR-10 at the largest batch size in the $\ell_2$ specialization); and its fixed-threshold bound retains a clipping floor.
Our language-modeling experiments are small-scale pretraining from scratch; private fine-tuning of pretrained models \citep{li2022largelanguagemodelsstrong, kim2026dpmuon} is a natural next step.

\section*{Acknowledgements}

The research reported in this publication was supported by funding from King Abdullah University of Science and Technology (KAUST): i) KAUST Baseline Research Scheme, ii) CRG Grant ORFSCRG12-2024-6460, and iii) Center of Excellence for Generative AI, under award number 5940.

\bibliography{sample}

\newpage

\appendix

\tableofcontents

\section{Prerequisites}
\label{app:prereq}

Throughout the appendix, Assumptions~\ref{assum:smoothness} and~\ref{assum:stochastic_oracle} from Section~\ref{sec:setup} hold.

\begin{definition}[Dual norm]
\label{def:dual_norm}
Let $(\mathcal{X},\|\cdot\|)$ be a finite-dimensional normed vector space with dual space $\mathcal X^\star$. The dual norm is
\begin{equation*}
\|v\|_\star := \sup_{\|d\| \le 1} \langle v, d \rangle, \quad v \in \mathcal{X}^\star.
\end{equation*}
\end{definition}

\begin{definition}[Linear minimization oracle (LMO)]
\label{def:lmo}
For $\mathcal{C}=\{x \in \mathcal{X} : \|x\|\leq1\}$, $\operatorname{LMO}(v) := \arg\min_{d \in \mathcal{C}} \langle v, d \rangle$ for $v \in \mathcal{X}^\star$; when the minimizer is not unique, any minimizer may be used.
\end{definition}

\begin{corollary}
\label{cor:lmo_basic_equality}
For any $d\in\operatorname{LMO}(v)$, Definitions~\ref{def:dual_norm} and~\ref{def:lmo} give $\langle v, d \rangle = -\|v\|_\star$; moreover, $\operatorname{LMO}(\lambda v)=\operatorname{LMO}(v)$ for every $\lambda>0$.
\end{corollary}

\begin{definition}[Norm-equivalence constants]
\label{assum:norm_equivalence}
Since $\mathcal X$ is finite-dimensional, there exist $\rho_1,\rho_2>0$ such that
\begin{equation}
    \|v\|_\star\le\rho_1\|v\|_2,
    \qquad
    \|v\|_2\le\rho_2\|v\|_\star
    \qquad\forall v\in\mathcal X^\star.
\end{equation}
\end{definition}

Table~\ref{tab:lmo_norms} lists the LMO, the constants $\rho_1,\rho_2$, and the order of $\mathbb E\|Z\|_\star$ for $Z\sim\mathcal N(0,I)$ (which enters the privacy-noise term) for common geometries.

\begin{table}[h]
\centering
\small
\caption{Common LMO geometries on $\mathbb R^d$ (or on $\mathbb R^{m\times n'}$ for the spectral norm).}
\label{tab:lmo_norms}
\begin{tabular}{lllccc}
\toprule
\textbf{Norm} $\|\cdot\|$ & \textbf{Method} & $\operatorname{LMO}(g)$ & $\rho_1$ & $\rho_2$ & $\mathbb E\|Z\|_\star$ \\
\midrule
Spectral & Muon & $-UV^\top$ for $g=U\Sigma V^\top$ & $\sqrt{\min\{m,n'\}}$ & $1$ & $\asymp\min\{m,n'\}\sqrt{\max\{m,n'\}}$ \\
$\ell_\infty$ & sign descent & $-\operatorname{sign}(g)$ & $\sqrt d$ & $1$ & $d\sqrt{2/\pi}$ \\
$\ell_2$ & normalized GD & $-g/\|g\|_2$ & $1$ & $1$ & $\approx\sqrt d$ \\
$\ell_1$ & coordinate descent & $-\operatorname{sign}(g_j)e_j$, $j\in\arg\max_i|g_i|$ & $1$ & $\sqrt d$ & $\asymp\sqrt{\log d}$ \\
\bottomrule
\end{tabular}
\end{table}

By Assumption~\ref{assum:smoothness}, for all $x, y \in \mathcal{X}$,
\begin{equation}
\label{eq:smoothtness_bound}
f(y)\le f(x) + \langle \nabla f(x),y-x\rangle + \frac{L}{2}\|y-x\|^2.
\end{equation}

\begin{corollary}
\label{cor:stoch_noise_prop}
For every $x$, write $\nabla f_{i}(x) = \nabla f(x) + \xi_i$ with $i \sim \mathrm{Unif}[n]$. By Assumption~\ref{assum:stochastic_oracle},
\begin{equation*}
    \mathbb{E}_i[\xi_i] = \frac{1}{n}\sum_{i=1}^{n}\xi_i = 0, \qquad \mathbb{E}_i[\|\xi_i\|^2_2] = \frac{1}{n}\sum_{i=1}^{n}\|\xi_i\|^2_2 \leq \sigma^2.
\end{equation*}
\end{corollary}

Define the clipping factor
\begin{equation*}
    a_c(v):=\frac{c}{\max\{c,\|v\|_2\}},
    \qquad\text{so that}\qquad
    \clip_c(v)=a_c(v)\,v.
\end{equation*}

\begin{corollary}
Since $v\mapsto\max\{c,\|v\|_2\}$ is $1$-Lipschitz,
\begin{equation}
\label{eq:multiplicative_lipschitz}
    |a_c(y)-a_c(x)|
    \le
    \frac{a_c(y)a_c(x)}{c}\|y-x\|_2.
\end{equation}
\end{corollary}

\section{DP-Muon convergence}
\label{app:convergence}

\subsection{Notation}

Let:
\begin{equation*}
    f(x) := \frac{1}{n}\sum_{i=1}^n f_i(x), \quad n=|\mathcal{D}|.
\end{equation*}
At iteration k, independently of previous sampling and Gaussian noise, draw
\begin{equation*}
    S_k \sim \text{Unif}\{S \subseteq [n] : |S|=B\}.
\end{equation*}
And compute private averaged gradient
\begin{equation*}
    \hat{g}_k = \frac{1}{B}\sum_{i \in S_k} \clip_c (\nabla f_i (x_k)) + z_k.
\end{equation*}

We define the clipped mini-batch average:
\begin{equation*}
    \bar g_k = \frac{1}{B} \sum_{i \in S_k} \clip_c (\nabla f_{i}(x_k)).
\end{equation*}

In the subsequent lemmas, we will use a simplified notation. Assuming a deterministic vector $g$ and a stochastic noise vector $\xi_i$ satisfying Corollary~\ref{cor:stoch_noise_prop}, we write:
\begin{equation*}
    \mu(g) := \frac{1}{n} \sum_{i=1}^n \clip_c(g + \xi_i).
\end{equation*}
We inject Gaussian privacy noise: 
\begin{equation*}
    z_k\sim \mathcal{N}(0,c^2\sigma^2_{\mathrm{DP}}I),\qquad \mathbb{E}[z_k]=0,\qquad 
\mathbb{E}\|z_k\|_2^2 =: dc^2\sigma^2_{\mathrm{DP}}.
\end{equation*}

From the algorithm definition, the overall oracle is:
\begin{equation}\label{eq:private-oracle}
\hat g_k := \bar g_k + z_k.
\end{equation}

For notational convenience, define
\begin{equation}
\label{eq:effective_sigma_bias}
    \widetilde\sigma:=\rho_1\sigma,
    \qquad
    \widetilde b_c
    :=
    \rho_1\min\left\{\sigma,\frac{\sigma^2}{c}\right\}.
\end{equation}

Define the pre-batch filtration representing the complete history of the algorithm prior to drawing the batch $S_t$ and privacy noise $z_t$ ($x_t$ is determined):
\begin{equation}
    \mathcal{F}_t := \sigma(x_0, \{S_s, z_s\}_{s=0}^{t-1}).
\end{equation}

\subsection{Useful lemmas}
\label{app:projective_lmo_lemma}

Lemma~\ref{lem:ray_main} in the main text combines Lemmas~\ref{lem:clipping_ray} and~\ref{lem:clipping_relative_variance} below, and Lemma~\ref{lem:projective_main} is Lemma~\ref{lem:projective_lmo}.

Throughout this subsection $\xi$ denotes $\xi_i$ with $i \sim \text{Unif}[n]$, and $\mathbb{E}$, $\mathbb{P}$ are with respect to i; by Corollary~\ref{cor:stoch_noise_prop}, $\mathbb{E} \xi = 0$ and $\mathbb{E} \|\xi\|_2^2 \leq \sigma^2$.

For $m,g\in\mathcal X^\star$, define the error
\begin{equation}
\label{eq:projective_error}
    \mathcal R(m,g)
    :=
    \inf_{\lambda>0}
    \frac{\|m-\lambda g\|_\star}{\lambda}.
\end{equation}

\begin{lemma}[Inner Product Bound]
\label{lem:projective_lmo}
Let $d\in\operatorname{LMO}(m)$. Then
\begin{equation}
\label{eq:projective_lmo_bound}
    \langle g,d\rangle
    \le
    -\|g\|_\star
    +
    2\mathcal R(m,g).
\end{equation}
\end{lemma}

\begin{proof}
Fix any $\lambda>0$. By
Corollary~\ref{cor:lmo_basic_equality},
\begin{align*}
    \lambda\langle g,d\rangle
    &=
    \langle m,d\rangle
    -
    \langle m-\lambda g,d\rangle
\\
    &\le
    -\|m\|_\star
    +
    \|m-\lambda g\|_\star
\\
    &\le
    -\lambda\|g\|_\star
    +
    2\|m-\lambda g\|_\star.
\end{align*}
Divide by $\lambda$ and take the infimum over $\lambda>0$.
\end{proof}

\begin{lemma}[Clipping Bias Smooth Bound]
\label{lem:clipping_ray}
For every $g$, there exists a scalar $\lambda(g)>0$ such
that
\begin{equation}
\label{eq:lambda_comparable}
    \frac{1}{4}a_{c}(g)
    \le
    \lambda(g)
    \le
    2a_{c}(g),
\end{equation}
and
\begin{equation*}
    \|\mu(g)-\lambda(g)g\|_2
    \le
    16\lambda(g)
    \min\left\{\sigma,\frac{\sigma^2}{c}\right\}.
\end{equation*}
Consequently,
\begin{equation*}
    \|\mu(g)-\lambda(g)g\|_\star
    \le
    16\lambda(g)\widetilde b_c.
\end{equation*}
\end{lemma}

\begin{proof}
\textbf{Case 1}: $c\ge\sigma$.
Since
\begin{equation*}
    \max\{c,\|g+\xi\|_2\}
    \le
    \max\{c,\|g\|_2\}
    +
    \|\xi\|_2,
\end{equation*}
we have
\begin{equation*}
    a_c(g+\xi)
    =
    \frac{c}{\max\{c,\|g+\xi\|_2\}}
    \ge
    \frac{c}{\max\{c,\|g\|_2\} + \|\xi\|_2}
    =
    \frac{a_{c}(g)}{1+a_{c}(g)\|\xi\|_2/c}.
\end{equation*}
The map $t\mapsto(1+t)^{-1}$ is convex on $[0,\infty)$, hence by Jensen's inequality:
\begin{equation*}
    \mathbb E[a_c(g+\xi)]
    \ge
    \frac{a_{c}(g)}{
        1+a_{c}(g)\mathbb E\|\xi\|_2/c
    }
    \ge
    \frac{a_{c}(g)}{2},
\end{equation*}
where we used
$\mathbb E\|\xi\|_2\le\sigma\le c$.

Using $\mathbb E[\xi]=0$ and
\eqref{eq:multiplicative_lipschitz}:
\begin{equation}
\label{eq:weighted_noise_large_c}
    \left\|
        \mathbb E[a_c(g+\xi)\xi]
    \right\|_2
    =
    \left\|
        \mathbb E[
            (a_c(g+\xi)-a_{c}(g))\xi
        ]
    \right\|_2
    \le
    \frac{a_{c}(g)}{c}
    \mathbb E[
        a_c(g+\xi)\|\xi\|_2^2
    ]
    \le
    a_{c}(g)\frac{\sigma^2}{c}.
\end{equation}

Suppose first that
\begin{equation*}
    \|g\|_2\ge\frac{4\sigma^2}{c}.
\end{equation*}
Define
\begin{equation*}
    \lambda(g)
    :=
    \frac{\langle \mu(g),g/\|g\|_2\rangle}{\|g\|_2} = \frac{\langle \mu(g),g\rangle}{\|g\|^2_2}.
\end{equation*}
Since
\begin{equation*}
    \mu(g)
    =  \mathbb{E} \big[ a_c(g+\xi)(g+\xi) \big]
    =
    \mathbb E[a_{c}(g+\xi)]g
    +
    \mathbb E[a_{c}(g+\xi)\xi],
\end{equation*}
previously proven bound $\mathbb E[a_c(g+\xi)]\ge \frac{a_{c}(g)}{2}$, $\left\|
        \mathbb E[a_c(g+\xi)\xi]
    \right\|_2 \le
    a_{c}(g)\frac{\sigma^2}{c}$ and Cauchy–Schwarz imply:
\begin{equation*}
    \lambda(g)
    =
    \mathbb{E}[a_c(g+\xi)] + \frac{\langle \mathbb{E}[a_c(g+\xi)\xi], g \rangle}{\|g\|_2^2}
    \ge
    \frac{a_{c}(g)}{2} - \frac{\|\mathbb{E}[a_c(g+\xi)\xi]\|_2 \|g\|_2}{\|g\|_2^2}
    \ge
    \frac{a_{c}(g)}{2}
    -
    \frac{a_{c}(g)\sigma^2}{c\|g\|_2}
    \ge
    \frac{a_{c}(g)}{4}.
\end{equation*}
If $\|g\|_2\ge c$, then
$\|\mu(g)\|_2\le c=a_{c}(g)\|g\|_2$, so
$\lambda(g)\le\frac{\|\mu(g)\|_2}{\|g\|_2}\le \frac{c}{\|g\|_2}=a_{c}(g)$. If $\|g\|_2<c$, then $a_{c}(g)=1$ and:
\begin{equation*}
    \lambda(g)
    =
    \mathbb{E}[a_c(g+\xi)] + \frac{\langle \mathbb{E}[a_c(g+\xi)\xi], g \rangle}{\Vert{}g\Vert{}_2^2}
    \le
    1+\frac{\sigma^2}{c\|g\|_2}
    \le
    \frac54 = \frac{5}{4}a_c(g).
\end{equation*}
Thus \eqref{eq:lambda_comparable} holds. By the definition of
$\lambda(g)$, the residual is orthogonal to $g$. Combining with \eqref{eq:weighted_noise_large_c} we have:
\begin{equation*}
    \|\mu(g)-\lambda(g)g\|_2
    \le
    \left\|
        \mathbb E[a_c(g+\xi)\xi]
    \right\|_2
    \le
    a_{c}(g)\frac{\sigma^2}{c}
    \le
    4\lambda(g)\frac{\sigma^2}{c}.
\end{equation*}

Now suppose that
\begin{equation*}
    \|g\|_2<\frac{4\sigma^2}{c},
\end{equation*}
and choose $\lambda(g):=a_{c}(g)$. Then
\begin{equation*}
    \mu(g)-a_{c}(g)g
    =
    \bigl(
        \mathbb E[a_c(g+\xi)]-a_{c}(g)
    \bigr)g
    +
    \mathbb E[a_c(g+\xi)\xi].
\end{equation*}
The second term is bounded by
$a_{c}(g)\sigma^2/c$. For the first term using \eqref{eq:multiplicative_lipschitz},
\begin{align*}
    \|g\|_2\,
    \mathbb E|a_c(g+\xi)-a_{c}(g)|
    \le
    \frac{a_{c}(g)\|g\|_2}{c}
    \mathbb E[
        a_c(g+\xi)\|\xi\|_2
    ]
    \le
    a_{c}(g)\frac{\|g\|_2\sigma}{c}
    \le
    4a_{c}(g)\frac{\sigma^2}{c},
\end{align*}
where the last step uses both
$\|g\|_2<4\sigma^2/c$ and $\sigma/c\le1$. Hence
\[
    \|\mu(g)-a_{c}(g)g\|_2
    \le
    5a_{c}(g)\frac{\sigma^2}{c}.
\]
\textbf{Case 2}: $c<\sigma$.
If $\|g\|_2<8\sigma$, choose $\lambda(g):=a_{c}(g)$. Since
\[
    \|\mu(g)\|_2\le c,
    \qquad
    \|a_{c}(g)g\|_2\le c,
\]
we obtain
\[
    \|\mu(g)-a_{c}(g)g\|_2
    \le2c.
\]
If $\|g\|_2\le c$, then $a_{c}(g)=1$ and
$2c\le2\sigma$. If $\|g\|_2>c$, then $a_{c}(g)=c/\|g\|_2$ and
$\|g\|_2<8\sigma$, which gives
\[
    2c
    \le
    16a_{c}(g)\sigma.
\]
It remains to consider $\|g\|_2\ge8\sigma$. Define the good event
\[
    \mathcal G
    :=
    \{\|\xi\|_2\le \|g\|_2/2\}.
\]
By Markov's inequality,
\[
    q:=\mathbb P(\mathcal G^c)
    \le
    \frac{4\sigma^2}{\|g\|_2^2}
    \le
    \frac1{16}.
\]
On $\mathcal G$, we have
$\|g+\xi\|_2\in[\|g\|_2/2,3\|g\|_2/2]$ and
$\|g\|_2/2>c$. Therefore,
\[
    \left\langle
        \frac{g}{\|g\|_2},
        \clip_c(g+\xi)
    \right\rangle
    =
    c
    \frac{
        \|g\|_2+\langle g/\|g\|_2,\xi\rangle
    }{
        \|g+\xi\|_2
    }
    \ge
    \frac c3.
\]
On $\mathcal G^c$, the same inner product is at least $-c$.
Consequently,
\[
    \left\langle
        \mu(g),
        \frac{g}{\|g\|_2}
    \right\rangle
    \ge
    \frac c3(1-q)-cq
    =
    \frac c3(1-4q)
    \ge
    \frac c4.
\]
Set
\[
    \lambda(g)
    :=
    \frac{1}{\|g\|_2}
    \left\langle
        \mu(g),
        \frac{g}{\|g\|_2}
    \right\rangle.
\]
Then
\[
    \frac{a_{c}(g)}{4}
    \le
    \lambda(g)
    \le
    a_{c}(g).
\]
Let $P_{g^\perp}$ be the Euclidean projection onto the orthogonal
complement of $g$. On $\mathcal G$,
\[
    \left\|
        P_{g^\perp}
        \clip_c(g+\xi)
    \right\|_2
    \le
    \frac{2c}{\|g\|_2}\|\xi\|_2,
\]
whereas on $\mathcal G^c$ it is at most $c$. Hence
\begin{align*}
    \|P_{g^\perp}\mu(g)\|_2
    &\le
    \frac{2c\sigma}{\|g\|_2}
    +
    cq
\\
    &\le
    \frac{5c\sigma}{2\|g\|_2}
    =
    \frac52a_{c}(g)\sigma
    \le
    10\lambda(g)\sigma.
\end{align*}
By construction,
$\mu(g)-\lambda(g)g=P_{g^\perp}\mu(g)$.

\textbf{Case} $g=0$.
Choose $\lambda(0):=1$. If $c<\sigma$, then
\[
    \|\mu(0)\|_2
    \le
    \mathbb E[
        \min\{\|\xi\|_2,c\}
    ]
    \le c<\sigma.
\]
If $c\ge\sigma$, then
\[
    \mu(0)
    =
    \mathbb E[
        \clip_c(\xi)-\xi
    ],
\]
and
\[
    \|\mu(0)\|_2
    \le
    \mathbb E[
        \|\xi\|_2
        \mathbf 1\{\|\xi\|_2>c\}
    ]
    \le
    \frac{\mathbb E\|\xi\|_2^2}{c}
    \le
    \frac{\sigma^2}{c}.
\]
This completes the proof.
\end{proof}
\begin{lemma}[Clipping Variance]
\label{lem:clipping_relative_variance}
With the notation of Lemma~\ref{lem:clipping_ray},
\begin{equation}
\label{eq:relative_clipping_variance}
    \mathbb E\left[
        \|
        \clip_c(g+\xi)
            -
            \mu(g)
        \|_2^2
    \right]
    \le
    20a_c(g)^2\sigma^2
    \le
    320\lambda(g)^2\sigma^2.
\end{equation}
\end{lemma}
\begin{proof}
Because the mean minimizes the mean-square error:
\begin{equation*}
    \mathbb E\|
        \clip_c(g+\xi)-\mu(g)
    \|_2^2
    \le
    \mathbb E\|
        \clip_c(g+\xi)
        -
        \clip_c(g)
    \|_2^2.
\end{equation*}

If $\|g\|_2\le 2c$, the clipping scalar satisfies $a_c(g)\ge 1/2$. Because Euclidean projection onto the radius-$c$ ball is non-expansive,
\begin{equation*}
    \mathbb E\|
        \clip_c(g+\xi)
        -
        \clip_c(g)
    \|_2^2
    \le
    \mathbb E\|\xi\|_2^2
    \le
    \sigma^2
    \le
    4a_c(g)^2\sigma^2.
\end{equation*}

Now suppose $\|g\|_2 > 2c$. Define the event where the noise is bounded relative to the gradient:
\begin{equation*}
    \mathcal G
    :=
    \big\{ \|\xi\|_2\le\|g\|_2/2 \big\}.
\end{equation*}

By the reverse triangle inequality and the definition of $\mathcal{G}$ ($\|\xi\|_2 \le \|g\|_2/2$), any vector $v$ on the line segment connecting $g$ and $g+\xi$ has a length of at least:
\begin{equation*}
    \|v\|_2 \ge \|g\|_2 - \|\xi\|_2 \ge \frac{\|g\|_2}{2} > c.
\end{equation*}
Thus, the entire segment remains strictly outside the clipping ball. 

Outside the ball, the clipping operator is the radial projection $\clip(v) = cv/\|v\|_2$. Its Jacobian matrix is:
\begin{equation*}
    J_{\clip}(v) = \frac{c}{\|v\|_2} \left( I - \frac{vv^\top}{\|v\|_2^2} \right).
\end{equation*}
The term $(I - vv^\top/\|v\|_2^2)$ is an orthogonal projection matrix onto the subspace perpendicular to $v$. Because orthogonal projections have a spectral norm of exactly $1$, the operator norm of the Jacobian is strictly determined by the scalar prefix:
\begin{equation*}
    \|J_{\clip}(v)\|_{2} = \frac{c}{\|v\|_2} \cdot 1 \le \frac{c}{\|g\|_2/2} = \frac{2c}{\|g\|_2} = 2a_c(g).
\end{equation*}

Finally, by the local Lipschitz property of the clipping operator, we bound the distance between the outputs using the maximum operator norm along the line segment connecting the inputs:
\begin{equation*}
    \|\clip_c(g+\xi) - \clip_c(g)\|_2 
    \le \left( \max_{v} \|J_{\clip}(v)\|_{2} \right) \|(g+\xi) - g\|_2 \le 2a_c(g)\|\xi\|_2.
\end{equation*}
On the complement event $\mathcal G^c$, the distance is bounded by the ball's diameter, $2c = 2a_c(g)\|g\|_2$. Applying Markov's inequality to $\|\xi\|_2^2$ yields:
\begin{equation*}
    \mathbb P(\mathcal G^c)
    =
    \mathbb P\Big(\|\xi\|_2^2 > \|g\|_2^2 / 4\Big)
    \le
    \frac{4\sigma^2}{\|g\|_2^2}.
\end{equation*}
Summing over both events gives:
\begin{align*}
    \mathbb E\|
        \clip_c(g+\xi)
        -
        \clip_c(g)
    \|_2^2
    &\le
    \mathbb E\big[ (2a_c(g)\|\xi\|_2)^2 \big]
    +
    (2a_c(g)\|g\|_2)^2 \, \mathbb P(\mathcal G^c) \\
    &\le
    4a_c(g)^2\sigma^2
    +
    4a_c(g)^2\|g\|_2^2
    \left(\frac{4\sigma^2}{\|g\|_2^2}\right) \\
    &=
    20a_c(g)^2\sigma^2.
\end{align*}
Finally, substituting the lower bound from \eqref{eq:lambda_comparable} implies $a_c(g)^2\le 16\lambda(g)^2$, concluding the proof.
\end{proof}

\subsection{Bias-corrected momentum and an exact master theorem}
\label{app:projective_master}

Define
\begin{equation*}
    \overline m_k
    :=
    \frac{m_{k+1}}{1-(1-\alpha)^{k+1}}
    =
    \sum_{t=0}^k
    w_{k,t}\widehat g_t,
    \quad
    w_{k,t}
    :=
    \frac{
        \alpha(1-\alpha)^{k-t}
    }{
        1-(1-\alpha)^{k+1}
    }.
\end{equation*}
Note that,
\begin{equation*}
    \mathrm{LMO}(m_{k+1})
    =
    \mathrm{LMO}(\overline m_k),
    \quad
    \sum_{t=0}^k w_{k,t}=1.
\end{equation*}

Denote 

\[
g_k = \nabla f(x_k), \quad d_k := LMO(m_{k+1}).
\]

For every $t$, define
\[
    \mu_t
    := \mu (g_t) = \frac{1}{n} \sum_{i=1}^{n} \clip_c(\nabla f_i(x_t)).
\]    

According to the Lemma~\ref{lem:clipping_ray}, we choose
$2a_{c}(g_t) \geq \lambda_t \geq \frac{1}{4}a_{c}(g_t) > 0$ and $r_t$ such that
\begin{equation}
\label{eq:mu_t_ray_decomposition}
    \mu_t
    =
    \lambda_t\nabla f(x_t)+r_t,
    \qquad
    \|r_t\|_\star
    \le
    16\lambda_t\widetilde b_c.
\end{equation}

Define the scalar ray average
\begin{equation*}
    A_k
    :=
    \sum_{t=0}^k
    w_{k,t}\lambda_t
    >0.
\end{equation*}

Define the centered mini-batch fluctuation
\begin{equation}
\label{eq:centered_clipped_noise}
\varepsilon_t \;:=\; \frac{1}{B}\sum_{i \in S_t}
        \clip_c\big(\nabla f_i(x_t)\big) - \mu_t.
\end{equation}

Because $x_t$, $\lambda_t$, and $A_{t-1}$ are determined by $\mathcal{F}_t$, they are $\mathcal{F}_t$-measurable. Now we can derive the expectation and variance of $\varepsilon_t$ (a direct consequence of~\eqref{eq:relative_clipping_variance} and sampling without replacement):

\begin{equation}
\label{eq:centered_clipped_noise_properties}
        \mathbb E[\varepsilon_t\mid \mathcal{F}_t]=0,
    \qquad
    \mathbb E[
        \|\varepsilon_t\|_2^2
        \mid \mathcal{F}_t
    ]
    \le
    320\lambda_t^2\sigma^2 \frac{n-B}{B(n-1)} \leq \frac{320\lambda_t^2\sigma^2}{B}.
\end{equation}

Define the tracking terms

\begin{align*}
    D_k &:= \sum_{t=0}^k w_{k,t}\lambda_t(g_t-g_k)\\
    N_k &:= \sum_{t=0}^k w_{k,t}\varepsilon_t \\
    P_k &:= \sum_{t=0}^k w_{k,t}z_t
\end{align*}

\begin{theorem}[Exact projective convergence reduction]
\label{thm:projective_master}
Under Assumptions~\ref{assum:smoothness}--\ref{assum:stochastic_oracle},
the iterates of Algorithm \ref{alg:dpmuon_}
satisfy, for every $\gamma>0$, $\alpha\in(0,1)$, and $K\ge1$,
\begin{equation}
\label{eq:projective_master_bound}
\boxed{
    \frac1K\sum_{k=0}^{K-1}
    \mathbb E\|g_k\|_\star
    \le
    \frac{\Phi_0}{K\gamma}
    +
    \frac{L\gamma}{2}
    +
    32\widetilde b_c
    +
    \frac2K\sum_{k=0}^{K-1} 
    \mathbb E\big[\frac{\|D_k\|_{\star}}{A_k}
         + \frac{\|N_k\|_{\star}}{A_k} + \frac{\|P_k\|_{\star}}{A_k}
    \big],
}
\end{equation}
where $\Phi_0:=f(x_0)-f^\star$.
\end{theorem}

\begin{proof}
By the $L$-smoothness of $f$ and the constraint $\|d_k\|\le1$, the single-step descent inequality is:
\begin{equation*}
    f(x_{k+1})
    \le
    f(x_k)
    +
    \gamma\langle g_k,d_k\rangle
    +
    \frac{L\gamma^2}{2}.
\end{equation*}

The update direction is generated by the scale-invariant LMO, meaning $\operatorname{LMO}(m_{k+1})=\operatorname{LMO}(\overline m_k)$. Applying the projective LMO property (Lemma~\ref{lem:projective_lmo}) to this direction yields:
\begin{equation*}
    \langle g_k,d_k\rangle
    \le
    -\|g_k\|_\star
    +
    2\mathcal R(\overline m_k,g_k).
\end{equation*}

To bound this projective error, we expand the bias-corrected momentum. Using \eqref{eq:mu_t_ray_decomposition} and \eqref{eq:centered_clipped_noise}, we decompose it as:
\begin{align*}
    \overline m_k &= \sum_{t=0}^k w_{k,t} ( \bar g_t + z_t ) \\
    &= \sum_{t=0}^k w_{k,t} \Big( \lambda_t g_t + r_t + \varepsilon_t + z_t \Big) \\
    &= \sum_{t=0}^k w_{k,t} \lambda_t (g_t - g_k) + A_k g_k \\
    &+ \sum_{t=0}^k w_{k,t} r_t + \sum_{t=0}^k w_{k,t}\varepsilon_t + \sum_{t=0}^k w_{k,t}z_t.
\end{align*}

The clipping-ray residual satisfies
\begin{equation*}
   \frac{\|\sum_{t=0}^k w_{k,t}r_t\|_\star}{A_k} \le 16\widetilde b_c.
\end{equation*}

By taking the specific scaling factor $\lambda=A_k$ in \eqref{eq:projective_error}, we can bound the projective error by the sum of these normalized component norms:
\begin{equation*}
    \mathcal R(\overline m_k,g_k)
    \le
    16\widetilde b_c
    + \frac{\|D_k\|_{\star}}{A_k}
         + \frac{\|N_k\|_{\star}}{A_k} + \frac{\|P_k\|_{\star}}{A_k}.
\end{equation*}

Substituting this back into the LMO bound gives:
\begin{equation*}
    \mathbb{E}\langle g_k,d_k\rangle \le - \mathbb{E}\|g_k\|_\star + 32\widetilde b_c + 2\mathbb E\big[
        \frac{\|D_k\|_{\star}}{A_k}
         + \frac{\|N_k\|_{\star}}{A_k} + \frac{\|P_k\|_{\star}}{A_k}\big].
\end{equation*}

Plugging this inner product bound into the initial smoothness inequality yields:
\begin{equation*}
    f(x_{k+1})
    \le
    f(x_k)
    -
    \gamma\|g_k\|_\star
    +
    \gamma\Big(
        32\widetilde b_c
        +
        2\big[  \frac{\|D_k\|_{\star}}{A_k}
         + \frac{\|N_k\|_{\star}}{A_k} + \frac{\|P_k\|_{\star}}{A_k}\big]
    \Big)
    +
    \frac{L\gamma^2}{2}.
\end{equation*}

Rearranging to isolate $\gamma\|g_k\|_\star$, summing from $k=0$ to $K-1$, taking expectations, using the global lower bound $f(x_K)\ge f^\star$, and dividing the entire inequality by $K\gamma$ proves \eqref{eq:projective_master_bound}.
\end{proof}

\subsection{A closed-form specialization}
\label{app:projective_closed_form}

This subsection gives one sufficient condition under which the exact
tracking terms in Theorem~\ref{thm:projective_master} have the familiar
momentum-LMO scaling. This condition is sufficient, not claimed to be
necessary. Define:

\begin{equation*}
    a_k
    :=
    \min\left\{1,\frac{c}{\|g_k\|_2}\right\}.
\end{equation*}

\begin{lemma}[$A_k$ Bound]
\label{lem:ray_scale_comparability}
Suppose
\begin{equation}
\label{eq:sufficient_scale_stability}
    \frac{\rho_2 L\gamma}{c}
    \le
    \frac{\alpha}{4}.
\end{equation}
Then, for all $k$,
\begin{equation}
\label{eq:A_comparable_a}
    \frac{a_k}{5}
    \le
    A_k
    \le
    \frac{8a_k}{3},
    \qquad
    \frac3{32}
    \le
    \frac{\lambda_k}{A_k}
    \le
    10.
\end{equation}
\end{lemma}

\begin{proof}
Since $\|d_k\| \le 1$,
\begin{equation*}
    \|g_k - g_{k-1}\|_2 \le \rho_2 \|g_k - g_{k-1}\|_\star \le \rho_2 L \gamma\|d_{k-1}\| \le \rho_2 L\gamma.
\end{equation*}
By the reverse triangle inequality, setting $h := \rho_2 L\gamma/c$ implies
\begin{equation*}
    \|g_k\|_2 - \|g_{k-1}\|_2 \le ch \quad \text{and} \quad \|g_{k-1}\|_2 - \|g_k\|_2 \le ch.
\end{equation*}
Because $ch \ge 0$, we can bound the maximum functions as follows:
\begin{align*}
    \max\{c, \|g_k\|_2\} &\le \max\{c, \|g_{k-1}\|_2 + ch\} \le \max\{c, \|g_{k-1}\|_2\} + ch, \\
    \max\{c, \|g_{k-1}\|_2\} &\le \max\{c, \|g_k\|_2 + ch\} \le \max\{c, \|g_k\|_2\} + ch.
\end{align*}
Since the scaling factor is defined as
\begin{equation*}
    a_k = \frac{c}{\max\{c, \|g_k\|_2\}},
\end{equation*}
we can expand the ratio $a_k / a_{k-1}$ and apply the upper bound derived above:
\begin{equation*}
    \frac{a_k}{a_{k-1}} = \frac{\max\{c, \|g_{k-1}\|_2\}}{\max\{c, \|g_k\|_2\}} \le \frac{\max\{c, \|g_k\|_2\} + ch}{\max\{c, \|g_k\|_2\}} = 1 + \frac{ch}{\max\{c, \|g_k\|_2\}}.
\end{equation*}
Because the denominator $\max\{c, \|g_k\|_2\}$ is strictly bounded below by $c$, the fraction is at most $ch/c = h$. Thus,
\begin{equation*}
    \frac{a_k}{a_{k-1}} \le 1 + h.
\end{equation*}
By exact symmetry, applying the bound for $\max\{c, \|g_k\|_2\}$ to the inverse ratio yields:
\begin{align*}
    \frac{a_{k-1}}{a_k} = \frac{\max\{c, \|g_k\|_2\}}{\max\{c, \|g_{k-1}\|_2\}} \le \frac{\max\{c, \|g_{k-1}\|_2\} + ch}{\max\{c, \|g_{k-1}\|_2\}} \le 1 + h.
\end{align*}
Inverting this result provides the lower bound, giving us the final two-sided inequality:
\begin{equation*}
    \frac{1}{1+h} \le \frac{a_k}{a_{k-1}} \le 1+h.
\end{equation*}
Hence, for $0\le t\le k$,
\begin{equation*}
    \frac{a_k}{(1+h)^{k-t}}
    \le
    a_t
    \le
    a_k(1+h)^{k-t}.
\end{equation*}
Using $\lambda_t\in[a_t/4,2a_t]$ (Lemma~\ref{lem:clipping_ray}), and the recursive bound $a_t \le a_k(1+h)^{k-t}$, we can rewrite the moving average by re-indexing the sum with $j = k-t$:
\begin{equation*}
    A_k = \frac{\sum_{t=0}^k (1-\alpha)^{k-t} \lambda_t}{\sum_{t=0}^k (1-\alpha)^{k-t}}
    \le \frac{\sum_{t=0}^k (1-\alpha)^{k-t} 2a_k(1+h)^{k-t}}{\sum_{t=0}^k (1-\alpha)^{k-t}} = 2a_k \frac{ \sum_{j=0}^k [(1-\alpha)(1+h)]^j }{\sum_{j=0}^k (1-\alpha)^j }.
\end{equation*}
To bound this ratio of finite geometric series, we compare it to the ratio of their infinite series limits. For positive $x,y \in (0,1)$, such that $x > y$, the relationship $(\sum_{j=0}^k x^j) / (\sum_{j=0}^k y^j) \le (1-y)/(1-x)$ holds. Applying this with $y = 1-\alpha$ and $x = (1-\alpha)(1+h)$ yields:
\begin{equation*}
    \frac{ \sum_{j=0}^k [(1-\alpha)(1+h)]^j }{ \sum_{j=0}^k (1-\alpha)^j } \le \frac{1 - (1-\alpha)}{1 - (1-\alpha)(1+h)} = \frac{\alpha}{\alpha - h + \alpha h}.
\end{equation*}
Because $h\le\alpha/4$ (as we assumed in the lemma statement) and $\alpha < 1$, we can find the maximum of this fraction by minimizing its denominator:
\begin{equation*}
    \alpha - h(1-\alpha) \ge \alpha - \frac{\alpha}{4}(1-\alpha) = \frac{3\alpha}{4} + \frac{\alpha^2}{4} \ge \frac{3\alpha}{4}.
\end{equation*}
Substituting this minimum back into the fraction gives the final upper bound:
\begin{equation*}
    \frac{\alpha}{\alpha-h+\alpha h} \le \frac{\alpha}{3\alpha/4} = \frac{4}{3}.
\end{equation*}
Similarly, we establish the lower bound using $\lambda_t \ge a_t/4$ and the inverse recursive bound $a_t \ge a_k(1+h)^{-(k-t)}$:
\begin{equation*}
    A_k \ge \frac{a_k}{4} \frac{ \sum_{j=0}^k [(1-\alpha)/(1+h)]^j }{ \sum_{j=0}^k (1-\alpha)^j }.
\end{equation*}
Here, the numerator's geometric base $z = (1-\alpha)/(1+h)$ is strictly less than the denominator's base $y = 1-\alpha$. Consequently, the finite series ratio is bounded below by its infinite counterpart $(1-y)/(1-z)$:
\begin{equation*}
    \frac{ \sum_{j=0}^k [(1-\alpha)/(1+h)]^j }{ \sum_{j=0}^k (1-\alpha)^j } \ge \frac{1 - (1-\alpha)}{1 - \frac{1-\alpha}{1+h}} = \frac{\alpha(1+h)}{\alpha+h}.
\end{equation*}
Consider the function $\phi(h) = \frac{\alpha(1+h)}{\alpha+h}$ on the interval $[0, \alpha/4]$. Its derivative with respect to $h$ is strictly negative, meaning the function is monotonically decreasing and must attain its absolute minimum at the upper boundary $h = \alpha/4$. Thus:
\begin{equation*}
    \frac{\alpha(1+h)}{\alpha+h} \ge \frac{\alpha(1+\alpha/4)}{\alpha+\alpha/4} = \frac{4+\alpha}{5} \ge \frac{4}{5}.
\end{equation*}
Plugging this scalar back into our initial expression yields:
\begin{equation*}
    A_k \ge \frac{a_k}{4} \left( \frac{4}{5} \right) = \frac{a_k}{5}.
\end{equation*}
Combining these bounds with
$\lambda_k\in[a_k/4,2a_k]$ (Lemma~\ref{lem:clipping_ray}) proves
\eqref{eq:A_comparable_a}.

\end{proof}

\subsection{Drift and sampling-noise bounds}

Recall that
\begin{equation*}
    w_{k,t}
    :=
    \frac{
        \alpha(1-\alpha)^{k-t}
    }{
        1-(1-\alpha)^{k+1}
    }.
\end{equation*}
We define
\begin{equation}
\label{eq:beta_w_relation}
    \beta_k
    :=
    \frac{\alpha}{
        1-(1-\alpha)^{k+1}
    } = w_{k,k}.
\end{equation}
Consequently, the weights satisfy the recursion:
\begin{equation}
\label{eq:w_beta_recursion}
    w_{k,t} = \left( \beta_t \prod_{i=t+1}^{k-1} (1-\beta_i) \right) (1-\beta_k) = w_{k-1,t} (1-\beta_k).
\end{equation}
Then
\[
    \overline m_0=\widehat g_0,
    \qquad
    \overline m_k
    = \frac{m_{k+1}}{1-(1-\alpha)^{k+1}} =
    (1-\beta_k)\overline m_{k-1}
    +
    \beta_k\widehat g_k.
\]
Moreover, $\beta_k\ge\alpha$, and after the run-in iterations
\begin{equation}
\label{eq:burnin_definition}
    k_{\rm b}
    :=
    \left\lceil
        \frac{\log2}{\alpha}
    \right\rceil,
\end{equation}
we have $\beta_k\le2\alpha$.

\noindent Let
\begin{equation*}
    N_k
    :=
    \sum_{t=0}^k
    w_{k,t}\varepsilon_t,
    \qquad
    P_k
    :=
    \sum_{t=0}^k
    w_{k,t}z_t,
\end{equation*}
and
\[
    D_k
    :=
    \sum_{t=0}^k
    w_{k,t}\lambda_t(g_t-g_k).
\]

\begin{lemma}[Useful Recursions]
\label{lem:anpd_recursions}

\begin{align*}
    A_k &= (1 - \beta_k)A_{k-1} + \beta_k \lambda_k, \\
    D_k &= (1 - \beta_k)D_{k-1} + (1-\beta_k)A_{k-1}(g_{k-1} - g_k), \\
    N_k &= (1 - \beta_k)N_{k-1} + \beta_k \varepsilon_k.
\end{align*}

\begin{proof}

From \eqref{eq:beta_w_relation} and \eqref{eq:w_beta_recursion}:
\begin{align*}
    A_k &= \sum_{t=0}^k w_{k,t} \lambda_t = \sum_{t=0}^{k-1} w_{k,t} \lambda_t + w_{k,k} \lambda_k = \sum_{t=0}^{k-1} (1-\beta_k) w_{k-1,t} \lambda_t + \beta_k \lambda_k \\
    &= (1-\beta_k)\left( \sum_{t=0}^{k-1} w_{k-1,t} \lambda_t \right) + \beta_k \lambda_k = (1-\beta_k)A_{k-1} + \beta_k \lambda_k, \\
    D_k &= \sum_{t=0}^{k} w_{k,t}\lambda_t(g_t - g_k) = \sum_{t=0}^{k-1} w_{k,t}\lambda_t(g_t - g_k) \\
    &\stackrel{\eqref{eq:w_beta_recursion}}{=} (1-\beta_k) \sum_{t=0}^{k-1} w_{k-1,t}\lambda_t (g_t - g_{k-1} + g_{k-1} - g_k) \\
    &= (1-\beta_k) \sum_{t=0}^{k-1} w_{k-1,t}\lambda_t(g_t - g_{k-1}) + (1-\beta_k) \left( \sum_{t=0}^{k-1} w_{k-1,t}\lambda_t \right) (g_{k-1} - g_k) \\
    &= (1-\beta_k)D_{k-1} + (1-\beta_k)A_{k-1}(g_{k-1} - g_k), \\
    N_k &= \sum_{t=0}^k w_{k,t} \varepsilon_t = \sum_{t=0}^{k-1} w_{k,t} \varepsilon_t + w_{k,k} \varepsilon_k = \sum_{t=0}^{k-1} (1-\beta_k) w_{k-1,t} \varepsilon_t + \beta_k \varepsilon_k = (1-\beta_k)N_{k-1} + \beta_k \varepsilon_k.
\end{align*}
    
\end{proof}
    
\end{lemma}

\begin{lemma}[Drift and sampling-noise bounds]
\label{lem:closed_form_tracking}
Under \eqref{eq:sufficient_scale_stability}, there exists a universal
constant $C_{\rm stoch}>0$ such that
\begin{equation}
\label{eq:drift_closed_form}
    \frac{\|D_k\|_\star}{A_k}
    \le
    \frac{32L\gamma}{3\alpha},
\end{equation}
and, for
$T:=K-k_{\rm b}>0$,
\begin{equation}
\label{eq:stochastic_closed_form}
    \frac1T
    \sum_{k=k_{\rm b}}^{K-1}
    \mathbb E
    \left[
        \frac{\|N_k\|_\star}{A_k}
    \right]
    \le
    C_{\rm stoch}
    \frac{\widetilde\sigma}{\sqrt B}
    \left(
        \sqrt\alpha
        +
        \frac1{\alpha T}
    \right).
\end{equation}
\end{lemma}

\begin{proof}

 Consider the recursions of Lemma~\ref{lem:anpd_recursions}:

\begin{equation*}
    D_k = (1-\beta_k)D_{k-1} + (1-\beta_k)A_{k-1}(g_{k-1} - g_k).
\end{equation*}

\begin{equation*}
    A_k = (1-\beta_k)A_{k-1} + \beta_k \lambda_k.
\end{equation*}
Now we can define:
\begin{align*}
    q_k
    :=
    \frac{
        (1-\beta_k)A_{k-1}
    }{
        A_k
    }
    =
    1-\beta_k\frac{\lambda_k}{A_k}.
\end{align*}
By
\eqref{eq:A_comparable_a} and burn in ($\beta_k\ge\alpha$),
\begin{equation*}
    0\le q_k
    \le
    1-\frac{3\alpha}{32}
    =:\overline q.
\end{equation*}
Therefore,
\begin{align*}
    \frac{\|D_k\|_\star}{A_k} &= \frac{\|(1-\beta_k)D_{k-1} + (1-\beta_k)A_{k-1}(g_{k-1} - g_k)\|_\star}{A_k} \\
    &\le q_k \frac{\|D_{k-1}\|_\star}{A_{k-1}} + q_k \|g_{k-1} - g_k\|_\star \\
    &\le \overline q
    \frac{\|D_{k-1}\|_\star}{A_{k-1}}
    + \overline q L\gamma \\
    &\le \sum_{i=1}^{k}\overline{q}^i L \gamma + \overline{q}^{k}\frac{\|D_{0}\|_\star}{A_{0}} \\
    &\le L \gamma \frac{1}{1-\overline{q}} + \overline{q}^{k}\frac{\|D_{0}\|_\star}{A_{0}}. 
\end{align*}
Since $D_0=0$, summing the geometric series proves
\eqref{eq:drift_closed_form}.

\noindent For the sampling noise, Lemma~\ref{lem:anpd_recursions} gives
\begin{equation*}
    \frac{N_k}{A_k} = \frac{(1-\beta_k)N_{k-1} + \beta_k \varepsilon_k}{A_k} = \frac{(1-\beta_k)A_{k-1}}{A_k} \frac{N_{k-1}}{A_{k-1}} + \frac{\beta_k}{A_k} \varepsilon_k = q_k \frac{N_{k-1}}{A_{k-1}} + \frac{\beta_k}{A_k} \varepsilon_k.
\end{equation*}
Next, we expand the squared $L_2$ norm of this expression:
\begin{equation*}
    \frac{\|N_k\|_2^2}{A_k^2} = \left\| q_k \frac{N_{k-1}}{A_{k-1}} + \frac{\beta_k}{A_k} \varepsilon_k \right\|_2^2 = q_k^2 \frac{\|N_{k-1}\|_2^2}{A_{k-1}^2} + \frac{\beta_k^2}{A_k^2} \|\varepsilon_k\|_2^2 + 2 q_k \frac{\beta_k}{A_k} \left\langle \frac{N_{k-1}}{A_{k-1}}, \varepsilon_k \right\rangle.
\end{equation*}
We now take the conditional expectation:
\begin{align*}
    \mathbb{E}\left[ \left\langle \frac{N_{k-1}}{A_{k-1}}, \varepsilon_k \right\rangle \mathrel{\Bigg|} \mathcal{F}_k \right] &= \left\langle \frac{N_{k-1}}{A_{k-1}}, \mathbb{E}[\varepsilon_k \mid \mathcal{F}_k] \right\rangle = 0.
\end{align*}
This factorization is valid because the past accumulated noise $N_{k-1}$ and the denominator $A_{k-1}$ are $\mathcal{F}_k$-measurable.

Hence using $\mathbb E[\|\varepsilon_k\|_2^2 \mid \mathcal{F}_k] \le \frac{320\lambda_k^2\sigma^2}{B}$ \eqref{eq:centered_clipped_noise_properties} and
$\lambda_k/A_k\le10$ \eqref{eq:A_comparable_a}:
\begin{equation*}
    \mathbb E\left[
        \frac{\|N_k\|_2^2}{A_k^2}
        \,\middle|\,
        \mathcal{F}_k
    \right]
    \le
    q_k^2
    \frac{\|N_{k-1}\|_2^2}{A_{k-1}^2}
    + \frac{\beta_k^2}{A_k^2} \frac{320\lambda_k^2\sigma^2}{B}
    \le
    q_k^2
    \frac{\|N_{k-1}\|_2^2}{A_{k-1}^2}
    +
    32000\beta_k^2\frac{\sigma^2}{B},
\end{equation*}
Now using \eqref{eq:A_comparable_a} and $q_k\in[0,1]$,
\begin{equation*}
    q_k^2\le q_k
    =
    1-\beta_k\frac{\lambda_k}{A_k}
    \le
    1-\frac{3\beta_k}{32},
\end{equation*}
which is a deterministic bound. Let $X_k:=\mathbb E\bigl[\|N_k\|_2^2/A_k^2\bigr]$, where the expectation is over the mini-batch sampling (conditionally on the Gaussian noise sequence, which is independent of the batches; the argument is identical for the total expectation). Taking expectations of the previous conditional inequality gives
\begin{equation*}
    X_k
    \le
    \Bigl(1-\frac{3\beta_k}{32}\Bigr)X_{k-1}
    +
    32000\,\beta_k^2\frac{\sigma^2}{B},
    \qquad
    X_0=\mathbb E\frac{\|\varepsilon_0\|_2^2}{\lambda_0^2}\le\frac{320\sigma^2}{B}.
\end{equation*}
Since $\beta_k\le1$, induction yields the uniform bound $X_k\le C_0\sigma^2/B$ with $C_0:=32000\cdot\frac{32}{3}$. After the run-in, $\alpha\le\beta_k\le2\alpha$, so unrolling from $k_{\rm b}$ gives
\[
    X_k
    \le
    C_0
    \left(
        1-\frac{3\alpha}{32}
    \right)^{k-k_{\rm b}}
    \frac{\sigma^2}{B}
    +
    C_1\alpha\frac{\sigma^2}{B},
    \qquad
    C_1:=4\cdot32000\cdot\frac{32}{3},
    \qquad k\ge k_{\rm b}.
\]
By Jensen's inequality, $\|v\|_\star\le\rho_1\|v\|_2$, and $\sqrt{x + y} \le \sqrt{x} + \sqrt{y}$,
\[
\mathbb E\left[
        \frac{\|N_k\|_\star}{A_k}
    \right] \le \frac{\rho_1\sigma}{\sqrt{B}} \left( \sqrt{C_0} \left( 1-\frac{3\alpha}{32} \right)^{\frac{k-k_{\rm b}}{2}} + \sqrt{C_1}\sqrt{\alpha} \right).
\]
Summing over $k=k_{\rm b},\dots,K-1$ and using $\sum_{j\ge0}r^{j/2}=\frac{1}{1-\sqrt r}\le\frac{2}{1-r}=\frac{64}{3\alpha}$ for $r=1-\frac{3\alpha}{32}$ gives
\[
    \frac1T
    \sum_{k=k_{\rm b}}^{K-1}
    \mathbb E\left[
        \frac{\|N_k\|_\star}{A_k}
    \right]
    \le
    C_{\rm stoch}
    \frac{\rho_1\sigma}{\sqrt B}
    \left(
        \sqrt\alpha
        +
        \frac1{\alpha T}
    \right)
\]
for a universal $C_{\rm stoch}$. This is \eqref{eq:stochastic_closed_form}.
\end{proof}

\subsection{Privacy-noise bound}
\label{app:gaussian_dp_tracking}

Let $Z\sim\mathcal N(0,I)$. The map $z\mapsto\|z\|_\star$ is $\rho_1$-Lipschitz with respect to $\|\cdot\|_2$. The Gaussian component of the bias-corrected momentum is $P_k = \sum_{t=0}^k w_{k,t}z_t$. Because it is a linear combination of independent Gaussian vectors, $P_k$ is Gaussian with covariance $\operatorname{Cov}(P_k) = c^2\sigma_{\mathrm{DP}}^2 \big( \sum_{t=0}^k w_{k,t}^2 \big)I$.

To evaluate this variance multiplier, we explicitly sum the squares of the normalized exponential moving average weights $w_{k,t} = \frac{\alpha(1-\alpha)^{k-t}}{1-(1-\alpha)^{k+1}}$:
\begin{align*}
    \sum_{t=0}^k w_{k,t}^2
    &=
    \frac{\alpha^2}{(1-(1-\alpha)^{k+1})^2} \sum_{t=0}^k (1-\alpha)^{2(k-t)} \nonumber \\
    &=
    \frac{\alpha^2}{(1-(1-\alpha)^{k+1})^2} \frac{1-(1-\alpha)^{2(k+1)}}{1-(1-\alpha)^2} \nonumber \\
    &=
    \frac{\alpha^2}{(1-(1-\alpha)^{k+1})^2} \frac{(1-(1-\alpha)^{k+1})(1+(1-\alpha)^{k+1})}{\alpha(2-\alpha)} \nonumber \\
    &=
    \frac{\alpha}{2-\alpha}
    \frac{
        1+(1-\alpha)^{k+1}
    }{
        1-(1-\alpha)^{k+1}
    }.
\end{align*}
For $k\ge k_{\rm b} =\lceil \frac{\log2}{\alpha}\rceil$, we have $(1-\alpha)^{k+1}\le\frac12$, so the right-hand side is at most $3\alpha$. Therefore $P_k\overset{d}{=}s_kZ$ with $Z\sim\mathcal N(0,I)$ and $s_k\le\sqrt{3\alpha}\,c\sigma_{\mathrm{DP}}$, and $\|P_k\|_\star$ is stochastically dominated by $\sqrt{3\alpha}\,c\sigma_{\mathrm{DP}}\|Z\|_\star$.

To bound the dual norm of this noise, we apply the Tsirelson-Ibragimov-Sudakov concentration inequality. For the standard Gaussian vector $Z$ and the $\rho_1$-Lipschitz mapping $z \mapsto \|z\|_\star$, the one-sided tail bound is:
\begin{equation*}
    \mathbb{P}\Big( \|Z\|_\star - \mathbb E\|Z\|_\star \ge t \Big) \le \exp\left(-\frac{t^2}{2\rho_1^2}\right).
\end{equation*}
To guarantee a global failure probability of at most $\eta \in (0,1)$ across the $T$ tail iterations, we allocate a failure budget of $\eta/T$ to each step. Setting the right-hand side to $\eta/T$ yields the deviation threshold $t = \rho_1 \sqrt{2\log(T/\eta)}$. Scaling the expectation and the deviation threshold by the standard-deviation bound $\sqrt{3\alpha}\,c\sigma_{\mathrm{DP}}$ yields the momentum-filtered DP-noise scale:

\begin{equation*}
    H_{T,\eta}:= \mathbb E\|Z\|_\star
        +
        \rho_1\sqrt{
            2\log\left(\frac{T}{\eta}\right)
        }, \quad Z\sim\mathcal N(0,I),
\end{equation*}

\begin{equation}
\label{eq:nu_dp}
    \nu_{\rm DP}(\eta)
    :=
    c\sigma_{\mathrm{DP}}\sqrt{3\alpha}\,
    H_{T,\eta},
\end{equation}
This ensures that for any single step $k$, $\mathbb{P}\big( \|P_k\|_\star \ge \nu_{\rm DP}(\eta) \big) \le \frac{\eta}{T}$. Applying a union bound over the $T$ tail iterations guarantees that with probability at least $1-\eta$ over the Gaussian vectors $\{z_k\}$, the concentration holds:
\begin{equation}
\label{eq:uniform_dp_event}
    \|P_k\|_\star
    \le
    \nu_{\rm DP}(\eta)
    \qquad
    \forall k\in\{k_{\rm b},\ldots,K-1\}.
\end{equation}

Under \eqref{eq:A_comparable_a}, the smoothed normalization factor is bounded as $A_k \ge \frac{a_k}{5}$, which implies $\frac{1}{A_k} \le 5 \max\big\{1, \frac{\|g_k\|_2}{c}\big\}$. Consequently, on the event \eqref{eq:uniform_dp_event}, using the algebraic inequality $\max\{1, x\} \le 1 + x$ and the norm equivalence $\|g_k\|_2 \le \rho_2 \|g_k\|_\star$, we can bound:
\begin{align}
\label{eq:dp_projective_bound}
    \mathcal P_k
    =
    \frac{\|P_k\|_\star}{A_k}
    &\le
    5\nu_{\rm DP}(\eta)
    \max\left\{
        1,\frac{\|g_k\|_2}{c}
    \right\}
\nonumber\\
    &\le
    5\nu_{\rm DP}(\eta) \left( 1 + \frac{\|g_k\|_2}{c} \right)
\nonumber\\
    &\le
    5\nu_{\rm DP}(\eta)
    +
    5\rho_2
    \frac{\nu_{\rm DP}(\eta)}{c}
    \|g_k\|_\star.
\end{align}

\subsection{Closed-form high-probability bound}
\label{app:dp_muon_final_theorem}

The first $k_{\rm b}$ iterations can be controlled deterministically.
For every lower-bounded $L$-smooth function,
\begin{equation*}
    \|\nabla f(x)\|_\star^2
    \le
    2L(f(x)-f^\star).
\end{equation*}

Since every update has norm at most $\gamma$,
\[
    f(x_{k+1})-f^\star
    \le
    \left(
        \sqrt{f(x_k)-f^\star}
        +
        \gamma\sqrt{L/2}
    \right)^2.
\]
Therefore,
\begin{equation}
\label{eq:burnin_gap}
    f(x_{k_{\rm b}})-f^\star
    \le
    \overline\Phi
    :=
    \left(
        \sqrt{\Phi_0}
        +
        k_{\rm b}\gamma\sqrt{L/2}
    \right)^2.
\end{equation}

\begin{theorem}[Closed-form convergence of DP-Muon]
\label{thm:dp_muon_closed_form}
Assume
Assumptions~\ref{assum:smoothness}--\ref{assum:stochastic_oracle}
and the sufficient ray-scale condition
\eqref{eq:sufficient_scale_stability}. Let
$k_{\rm b}$ be defined by
\eqref{eq:burnin_definition}, and let
$T:=K-k_{\rm b}>0$.
Then, with probability at least $1-\eta$ over the Gaussian privacy
noise,
\begin{equation}
\label{eq:dp_muon_high_probability}
\boxed{
\begin{aligned}
    &
    \left(
        1-
        10\rho_2
        \frac{\nu_{\rm DP}(\eta)}{c} 
    \right)
    \frac1T
    \sum_{k=k_{\rm b}}^{K-1}
    \mathbb E_{S}\|g_k\|_\star
\\
    &\qquad\le
    \frac{\overline\Phi}{\gamma T}
    +
    \frac{L\gamma}{2}
    +
    \frac{64 L\gamma}{3\alpha}
    +
    32\widetilde b_c
\\
    &\qquad\quad+
    C_{\rm stoch}
    \frac{\widetilde\sigma}{\sqrt B}
    \left(
        \sqrt\alpha
        +
        \frac1{\alpha T}
    \right)
    +
    10\nu_{\rm DP}(\eta).
\end{aligned}
}
\end{equation}
Here $\mathbb E_S$ denotes expectation over stochastic-gradient
sampling conditional on the Gaussian-noise realization.
\end{theorem}

\begin{proof}
Apply Theorem~\ref{thm:projective_master} to the tail
$k=k_{\rm b},\ldots,K-1$, replacing the initial function gap by
\eqref{eq:burnin_gap}. Use
Lemma~\ref{lem:closed_form_tracking} for
$D_k$ and $N_k$. On the Gaussian event
\eqref{eq:uniform_dp_event}, use
\eqref{eq:dp_projective_bound}. The projective LMO lemma contributes a
factor $2$ in front of each tracking term, which yields the coefficient
$10\rho_2\nu_{\rm DP}(\eta)/c$ on the left and the additive
$10\nu_{\rm DP}(\eta)$ on the right. Summing the smoothness
inequality and dividing by $\gamma T$ proves
\eqref{eq:dp_muon_high_probability}.
\end{proof}

\begin{corollary}[Clean high-probability form]
\label{cor:dp_muon_clean_hp}
If, in addition,
\begin{equation*}
    10\rho_2
    \frac{\nu_{\rm DP}(\eta)}{c}
    \le
    \frac12.
\end{equation*}
then
\begin{align*}
    \frac1T
    \sum_{k=k_{\rm b}}^{K-1}
    \mathbb E_S\|g_k\|_\star
    =
    \mathcal O\Bigg(
        &
        \frac{\overline\Phi}{\gamma T}
        +
        \frac{L\gamma}{\alpha}
        +
        \widetilde b_c
        + \frac{\widetilde\sigma}{\sqrt B}
        \left(
            \sqrt\alpha
            +
            \frac1{\alpha T}
        \right)
        +
    c\sigma_{\mathrm{DP}}\sqrt{\alpha}\,
    H_{T,\eta}
    \Bigg).
\end{align*}
\end{corollary}

Theorem~\ref{thm:main} in the main text is Corollary~\ref{cor:dp_muon_clean_hp} with explicit constants: under $10\rho_2\nu_{\rm DP}(\eta)/c\le\frac12$, the left-hand side of~\eqref{eq:dp_muon_high_probability} is at least half the average, and $L\gamma+\frac{128L\gamma}{3\alpha}\le\frac{45L\gamma}{\alpha}$ for $\alpha\le1$.

\section{Rates}
\label{app:complexity}

Throughout this section, $H_{T,\eta}=\mathbb E\|Z\|_\star+\rho_1\sqrt{2\log(T/\eta)}$, so that $\nu_{\rm DP}(\eta)=\sqrt{3\alpha}\,c\,\sigma_{\rm DP}H_{T,\eta}$, and the two conditions of Theorem~\ref{thm:main} read
\begin{equation}
\label{eq:two_conditions}
    \text{(stability)}\quad \gamma\le\frac{c\,\alpha}{4\rho_2L},
    \qquad\qquad
    \text{(admissibility)}\quad \alpha\le\frac{1}{1200\,\rho_2^2\,\sigma_{\rm DP}^2H_{T,\eta}^2}.
\end{equation}
In addition, $T=K-k_{\rm b}>0$ requires $\alpha=\Omega(1/K)$.

\subsection{Fixed noise scale}
\label{sec:simplified_analysis}

Fix $c=\Theta(1)$ and $\sigma_{\rm DP}=\Theta(1)$, independent of $K$; then $\widetilde b_c=\Theta(1)$ is a constant floor.
Balancing $\overline\Phi/(\gamma T)$ against $L\gamma/\alpha$ gives $\gamma_\star=\sqrt{\overline\Phi\alpha/(LT)}$, and with $T\ge K/2$ the bound of Theorem~\ref{thm:main} becomes
\begin{equation}
\label{eq:complexity_alpha_simplified}
    \frac1T \sum_{k=k_{\rm b}}^{K-1} \mathbb E_S\|\nabla f(x_k)\|_\star
    =
    \mathcal O\Bigg(
        \sqrt{\frac{L\overline\Phi}{K\alpha}}
        +
        \widetilde b_c
        +
        \Bigl( \frac{\widetilde\sigma}{\sqrt B} + c\,\sigma_{\mathrm{DP}}H_{T,\eta} \Bigr) \sqrt\alpha
        +\frac{\widetilde\sigma}{\sqrt B\,\alpha K}
    \Bigg).
\end{equation}
Since $H_{T,\eta}=\Theta(\sqrt{\log K})$ for fixed dimension and $\eta$, balancing $1/\sqrt{K\alpha}$ against $\sqrt{\alpha\log K}$ gives
\begin{equation*}
    \alpha_\star = \Theta\left( \frac{1}{\sqrt{K \log K}} \right),
    \qquad
    \gamma_\star = \Theta\left( \frac{1}{K^{3/4} (\log K)^{1/4}} \right),
\end{equation*}
and the last term in~\eqref{eq:complexity_alpha_simplified} is then of lower order.
These choices are admissible for $K$ large enough: $\alpha_\star\gg1/K$; $\gamma_\star/\alpha_\star=\Theta(K^{-1/4}(\log K)^{1/4})\to0$, so the stability condition holds; and $\alpha_\star H_{T,\eta}^2=\Theta(\sqrt{\log K/K})\to0$, so the admissibility condition holds.
Explicitly, admissibility requires $\sqrt{K\log K}\gtrsim\rho_2^2\sigma_{\rm DP}^2H_{T,\eta}^2$, that is, $K\gtrsim\rho_2^4\sigma_{\rm DP}^4H_{T,\eta}^4$ up to logarithms, which is of order $d^2\sigma_{\rm DP}^4$ for the $\ell_2$ geometry and $\min\{m,n'\}^4\max\{m,n'\}^2\sigma_{\rm DP}^4$ for the spectral geometry; the rate below describes the regime beyond this horizon.
Moreover, $k_{\rm b}\gamma_\star\lesssim\gamma_\star/\alpha_\star\to0$, so $\overline\Phi\to\Phi_0$ and treating $\overline\Phi$ as a constant is self-consistent.
Substituting $\alpha_\star$ into~\eqref{eq:complexity_alpha_simplified} gives the first part of Corollary~\ref{cor:rates}:
\begin{equation*}
    \frac1T \sum_{k=k_{\rm b}}^{K-1} \mathbb E_S\|\nabla f(x_k)\|_\star
    =
    \mathcal{O}\left( \frac{(\log K)^{1/4}}{K^{1/4}} \right) + \mathcal O(\widetilde{b}_c).
\end{equation*}

\subsection{Noise calibrated to a fixed privacy budget}
\label{app:calibrated}

We now fix $(\varepsilon,\delta)$ and set $\sigma_{\rm DP}=c_2\sqrt{K\log(1/\delta)}/(n\varepsilon)$ with the constants of the fixed-size calibration~\eqref{eq:dp_noise}, so that Algorithm~\ref{alg:dpmuon_} is $(\varepsilon,\delta)$-DP whenever $\varepsilon<c_1q^2K$ (Proposition~\ref{prop:privacy}) and the noise scale grows with $K$.
Write $\ell_\delta:=\log(1/\delta)$ and
\[
    D:=\frac{c\,c_2\sqrt{\ell_\delta}\,H}{n\varepsilon},
    \qquad H:=\mathbb E\|Z\|_\star+\rho_1\sqrt{2\log(K/\eta)}\ \ge H_{T,\eta},
\]
so that $\nu_{\rm DP}(\eta)\le\sqrt3\,D\sqrt{\alpha K}$.
The key observation is that, with calibrated noise, the privacy term depends on $\alpha$ and $K$ only through the effective horizon $s:=\alpha K$.

\begin{corollary}[Calibrated rate]
\label{cor:calibrated}
Let $\Phi_0=f(x_0)-f^\star>0$, let $K$ be an integer, choose $s_\star:=\sqrt{L\Phi_0}/D$, $\alpha=s_\star/K$, and $\gamma=\sqrt{\Phi_0\alpha/(LK)}$, and suppose that
\begin{equation}
\label{eq:calibrated_regime}
    D\ \le\ \min\Bigl\{\frac{c^2}{1200\,\rho_2^2\sqrt{L\Phi_0}},\ \frac{\sqrt{L\Phi_0}}{4}\Bigr\}
\end{equation}
and $K\ge\max\{2s_\star,\ s_\star\widetilde\sigma^2/(B\sqrt{L\Phi_0}D)\}$, where $D$ and $s_\star$ depend on $K$ through $H$. Then, with probability at least $1-\eta$ over the privacy noise,
\[
    \frac1T \sum_{k=k_{\rm b}}^{K-1} \mathbb E_S\|\nabla f(x_k)\|_\star
    \ \lesssim\
    (L\Phi_0)^{1/4}\sqrt{D}\;+\;\frac{\widetilde\sigma\,D}{\sqrt{B\,L\Phi_0}}\;+\;\widetilde b_c,
\]
where $\lesssim$ hides a universal constant. In particular, the bound is $\mathcal O\bigl((L\Phi_0)^{1/4}\sqrt{c\,H\sqrt{\ell_\delta}/(n\varepsilon)}\bigr)+\mathcal O(\widetilde b_c)$, since the middle term is of order $1/(n\varepsilon)$.
\end{corollary}

\begin{proof}
Condition~\eqref{eq:calibrated_regime} gives $D\le\sqrt{L\Phi_0}/4$, that is, $s_\star\ge4$.

\emph{Run-in and initial gap.} Since $\alpha=s_\star/K\le1/2$, we have $k_{\rm b}\le\log2/\alpha+1\le2/\alpha=2K/s_\star\le K/2$, so $T\ge K/2$. Moreover, $k_{\rm b}\gamma\sqrt{L/2}\le\frac2\alpha\sqrt{\frac{\Phi_0\alpha}{LK}}\sqrt{\frac L2}=\sqrt{\frac{2\Phi_0}{s_\star}}\le\sqrt{\Phi_0}$, so $\overline\Phi\le4\Phi_0$.

\emph{Stability.} $\rho_2L\gamma/c\le\alpha/4$ is equivalent to $\sqrt{\alpha K}\ge4\rho_2\sqrt{L\Phi_0}/c$, that is, $s_\star\ge16\rho_2^2L\Phi_0/c^2$, which is equivalent to $D\le c^2/(16\rho_2^2\sqrt{L\Phi_0})$ and is implied by~\eqref{eq:calibrated_regime}.

\emph{Admissibility.} $10\rho_2\nu_{\rm DP}(\eta)/c\le10\sqrt3\rho_2D\sqrt{s_\star}/c\le\frac12$ is equivalent to $s_\star\le c^2/(1200\rho_2^2D^2)$, which is equivalent to $D\le c^2/(1200\rho_2^2\sqrt{L\Phi_0})$, as assumed.

\emph{Bound.} Theorem~\ref{thm:main} applies. With $\overline\Phi\le4\Phi_0$ and $T\ge K/2$,
\[
    \frac{2\overline\Phi}{\gamma T}+\frac{45L\gamma}{\alpha}
    \le\frac{16\Phi_0}{\gamma K}+\frac{45L\gamma}{\alpha}
    =(16+45)\sqrt{\frac{L\Phi_0}{\alpha K}}
    =61\,(L\Phi_0)^{1/4}\sqrt D,
\]
\[
    20\,\nu_{\rm DP}(\eta)\le20\sqrt3\,D\sqrt{s_\star}=20\sqrt3\,(L\Phi_0)^{1/4}\sqrt D,
\]
\[
    C\frac{\widetilde\sigma}{\sqrt B}\Bigl(\sqrt\alpha+\frac{1}{\alpha T}\Bigr)
    \le C\frac{\widetilde\sigma}{\sqrt B}\sqrt{\frac{s_\star}{K}}+\frac{2C\widetilde\sigma}{\sqrt B\,s_\star}
    \le C(L\Phi_0)^{1/4}\sqrt D+\frac{2C\widetilde\sigma D}{\sqrt{B\,L\Phi_0}},
\]
where the last step uses $K\ge s_\star\widetilde\sigma^2/(B\sqrt{L\Phi_0}D)$ for the first term and $s_\star=\sqrt{L\Phi_0}/D$ for the second. Adding $64\widetilde b_c$ completes the proof.
\end{proof}

\paragraph{Discussion.}
For the $\ell_2$ geometry, $\mathbb E\|Z\|_2\le\sqrt d$ and $\rho_1=\rho_2=1$, so the bound reads
\[
    \mathcal O\bigl((L\Phi_0)^{1/4}(c^2d\log(1/\delta))^{1/4}(n\varepsilon)^{-1/2}\bigr)+\mathcal O(\widetilde b_c)
\]
up to logarithmic factors in $K/\eta$, which matches the $(\sqrt d/(n\varepsilon))^{1/2}$ dependence of private normalized and clipped SGD \citep{yang2022normalized}.
Condition~\eqref{eq:calibrated_regime} requires $n\varepsilon\gtrsim H\sqrt{\log(1/\delta)}\max\{\rho_2^2\sqrt{L\Phi_0}/c,\ c/\sqrt{L\Phi_0}\}$ (with the constant $1200$ in the first term).
The privacy contribution to the bound is below the initial gradient bound $\sqrt{2L\Phi_0}$ as soon as $n\varepsilon\gtrsim cH\sqrt{\log(1/\delta)}/\sqrt{L\Phi_0}$ (the clipping floor $\widetilde b_c$ does not depend on $n\varepsilon$); for $c\asymp\sqrt{L\Phi_0}$ and the $\ell_2$ geometry both requirements reduce to $n\varepsilon\gtrsim\sqrt{d\log(1/\delta)}$, but for $c\ll\sqrt{L\Phi_0}$, that is, under strongly active clipping, the admissibility requirement is stronger by the factor $\rho_2^2L\Phi_0/c^2$.
In the experiments of Section~\ref{sec:experiments}, even with $\sqrt{L\Phi_0}/c=1$ and $H\approx\sqrt d$, the first term requires $n\varepsilon\gtrsim6\cdot10^{6}$ on CIFAR-10 against the actual $n\varepsilon=4\cdot10^{5}$.
Because $H$, and hence $D$, grow with $\log(K/\eta)$ (so that $s_\star=\sqrt{L\Phi_0}/D$ shrinks), condition~\eqref{eq:calibrated_regime} bounds the horizon $K$ from above at fixed $(n,\varepsilon,\delta)$: the admissible horizons are the integers $K$ satisfying the two lower bounds of the corollary, condition~\eqref{eq:calibrated_regime}, and the accountant condition $\varepsilon<c_1q^2K$ of Proposition~\ref{prop:privacy}. Since the dependence on $K$ is logarithmic, this range is very large whenever~\eqref{eq:calibrated_regime} holds with room, but it is not unrestricted.
Faster rates for private nonconvex optimization under Lipschitz losses use variance reduction or tree aggregation \citep{tran2022momentum, arora2023faster, lowy2024make}; obtaining them for momentum-LMO methods under active clipping is left for future work.

\section{Privacy Noise and Batch Size in the Experiments}
\label{app:scaling}

Throughout this appendix the privacy unit is one example (CIFAR-10) or one 1{,}024-token sequence (language modeling).
CIFAR-10 experiments use $n=50$k, 80 epochs, and $d\approx2$M. Language-modeling experiments use $n\approx5.86\cdot10^{5}$ sequences (a 600M-token budget, one epoch) and $d\approx30$M.
This appendix quantifies the privacy noise in our experiments (Section~\ref{app:scaling_accountant}), checks the conditions of Theorem~\ref{thm:main} against the tuned configurations (Section~\ref{app:scaling_batch}), and reports the spread between optimizers (Section~\ref{app:scaling_evidence}).

\subsection{The privacy noise as a function of the batch size}
\label{app:scaling_accountant}

With Poisson sampling at rate $q=B/n$, $E$ epochs, and hence $K=En/B$ steps, the noise multiplier $\sigma_{\rm mult}(B)$ needed for a target $(\varepsilon,\delta)$ is determined by the accountant, and the noise added to the averaged gradient is $z_k\sim\mathcal N(0,c^2\sigma_{\rm DP}^2I)$ with $\sigma_{\rm DP}=\sigma_{\rm mult}/B$.
We computed $\sigma_{\rm mult}(B)$ with the R\'enyi-DP accountant of the Poisson-subsampled Gaussian mechanism \citep{wang2019subsampled} (the default fractional and integer orders of Opacus and its conversion to $(\varepsilon,\delta)$) for $\varepsilon=8$ and $\delta=1/(2n)$.
Two regimes are relevant:
\begin{itemize}[leftmargin=1.2em,itemsep=1pt]
\item \emph{Large-multiplier regime} ($\sigma_{\rm mult}\gtrsim1$, many epochs): the accountant gives $\sigma_{\rm mult}\approx c_2\,q\sqrt{K\log(1/\delta)}/\varepsilon\propto\sqrt B$, so the noise on the average decays as $B^{-1/2}$, like sampling noise.
\item \emph{Small-multiplier regime} ($\sigma_{\rm mult}\lesssim1$, few epochs): $\sigma_{\rm mult}$ varies slowly with $B$, so the noise on the average decays almost as $B^{-1}$.
\end{itemize}
Table~\ref{tab:accountant} and Figure~\ref{fig:dp_noise_scaling} report the results.
Three facts stand out.
First, the root-mean-square norm of the privacy noise, $c\,\sigma_{\rm mult}\sqrt d/B$, exceeds the root-mean-square bound $c/\sqrt B$ on the sampling error of the clipped mini-batch average by a factor of 37--103 on image classification and 26--192 on language modeling. Because every clipped gradient has norm at most $c$, this bound holds both for fixed-size sampling, where $\mathbb E\|\bar g_k-\mu_k\|_2^2=\frac{n-B}{B(n-1)}\,\frac1n\sum_i\|\clip_c(\nabla f_i(x_k))-\mu_k\|_2^2\le c^2/B$, and for Poisson sampling normalized by the expected batch size, where the mean-square error equals $\frac{1-q}{B}\,\frac1n\sum_i\|\clip_c(\nabla f_i(x_k))\|_2^2\le c^2/B$. It bounds the mean-square error, not each realization: with $n=4$, $B=2$, and clipped scalar gradients $(c,c,-c,-c)$, the sample of the two positive gradients has probability $1/6$ and error $c>c/\sqrt2$. In root-mean-square terms the gradient noise is therefore set mainly by the privacy budget rather than by the data. This compares the noises themselves; it does not show that the sampling term of Theorem~\ref{thm:main}, which is expressed through the centered variance $\sigma$ of the unclipped gradients, is negligible, since $\sigma$ is not measured in our experiments. Second, the decay exponent $q_{\rm eff}:=-\,\mathrm d\log(\sigma_{\rm mult}/B)/\mathrm d\log B$ lies strictly between $1/2$ and $1$: from $0.87$ to $0.53$ on image classification and from $0.95$ to $0.88$ on language modeling.
Third, the observed optimal batch sizes lie where the norm of the privacy noise, $\|z_k\|_2\approx c\,\sigma_{\rm mult}\sqrt d/B$, is comparable to the clipping radius $c$, which upper-bounds the norm of the expected clipped gradient $\mu_k$: the ratio is $0.4$--$0.6$ on CIFAR-10 and $0.7$--$2.4$ on language modeling at the optima.

\begin{table}[h]
\centering
\small
\caption{Noise multiplier for $(\varepsilon,\delta)=(8,1/(2n))$ from the R\'enyi-DP accountant, ratio of the privacy-noise norm to the clipping radius ($\mathrm{NSR}=\sigma_{\rm mult}\sqrt d/B$), and ratio of the root-mean-square norm of the privacy noise to the root-mean-square bound $c/\sqrt B$ on the sampling error of clipped gradients (DP/samp.\ $=\sigma_{\rm mult}\sqrt{d/B}$). Since $c/\sqrt B$ is an upper bound, this ratio is a lower bound on the ratio of root-mean-square noise levels; it does not bound individual realizations. Bold: observed DP-optimal batch sizes (Figures~\ref{fig:cnn_baseline_loss_vs_bs} and~\ref{fig:babygpt_baseline_loss_vs_bs}). The row $B=2500$ for CIFAR-10 is included for the accountant only; it is not part of the DP grid of Appendix~\ref{app:cnn_details}.}
\label{tab:accountant}
\begin{tabular}{rrrrr@{\hspace{2.2em}}rrrrr}
\toprule
\multicolumn{5}{c}{CIFAR-10 ($n=50$k, 80 epochs, $d\approx2$M)} & \multicolumn{5}{c}{Language modeling ($n\approx586$k, 1 epoch, $d\approx30$M)} \\
$B$ & $K$ & $\sigma_{\rm mult}$ & NSR & DP/samp. & $B$ & $K$ & $\sigma_{\rm mult}$ & NSR & DP/samp. \\
\midrule
64 & 62500 & 0.58 & 12.9 & 103 & 128 & 4578 & 0.40 & 17.0 & 192 \\
128 & 31250 & 0.65 & 7.2 & 81 & 256 & 2289 & 0.42 & 8.9 & 142 \\
256 & 15625 & 0.74 & 4.1 & 65 & 512 & 1144 & 0.43 & 4.7 & 105 \\
512 & 7812 & 0.87 & 2.4 & 55 & \textbf{1024} & 572 & 0.46 & 2.4 & 78 \\
1024 & 3906 & 1.07 & 1.5 & 47 & \textbf{2048} & 286 & 0.48 & 1.3 & 59 \\
2500 & 1600 & 1.48 & 0.84 & 42 & \textbf{4096} & 143 & 0.52 & 0.69 & 44 \\
\textbf{5000} & 800 & 1.98 & 0.56 & 40 & 8192 & 72 & 0.55 & 0.37 & 34 \\
\textbf{10000} & 400 & 2.70 & 0.38 & 38 & 16384 & 36 & 0.60 & 0.20 & 26 \\
25000 & 160 & 4.14 & 0.23 & 37 & & & & & \\
\bottomrule
\end{tabular}
\end{table}

\begin{figure}[h]
\centering
\includegraphics[width=0.85\textwidth]{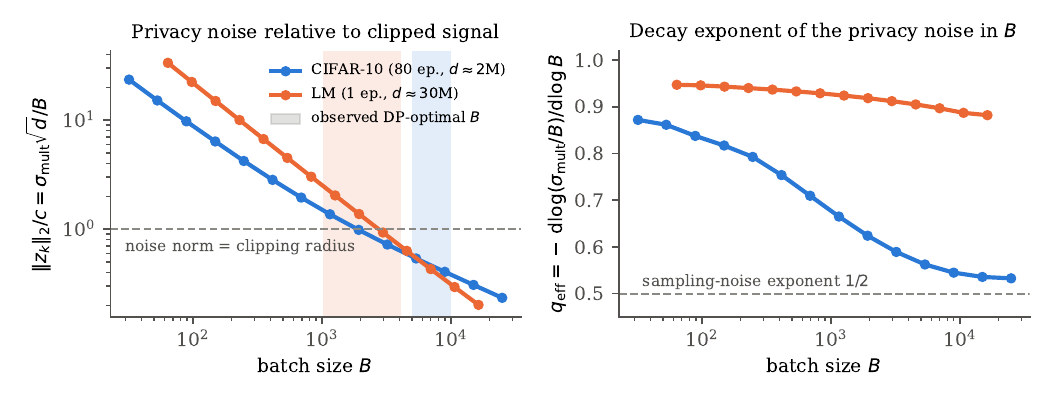}
\caption{Left: norm of the privacy noise added to the averaged gradient relative to the clipping radius, for the two experimental settings; shaded regions mark the observed DP-optimal batch sizes. Right: decay exponent of the privacy noise in the batch size; sampling noise has exponent $1/2$.}
\label{fig:dp_noise_scaling}
\end{figure}

\subsection{Applicability of Theorem~\ref{thm:main} in the experiments}
\label{app:scaling_batch}

With $\sigma_{\rm DP}=\sigma_{\rm mult}/B$, the noise condition in~\eqref{eq:main_conditions} reads
\begin{equation}
\label{eq:dp_admissible_batch}
    B\ \ge\ 20\sqrt3\,\rho_2\sqrt\alpha\,H\,\sigma_{\rm mult}(B),
\end{equation}
with $H=\mathbb E\|Z\|_\star+\rho_1\sqrt{2\log(T/\eta)}$; it is a lower bound on the batch size.
For the $\ell_2$ geometry ($H\approx\sqrt d$, $\rho_2=1$) it is equivalent to $\mathrm{NSR}=\sigma_{\rm mult}\sqrt d/B\le1/(20\sqrt3\,\sqrt\alpha)$, that is, $\mathrm{NSR}\le0.13$, $0.09$, and $0.065$ for $\alpha=0.05$, $0.1$, and $0.2$.
The observed optima sit at $\mathrm{NSR}\approx0.4$--$2.4$ (Table~\ref{tab:accountant}), about $4$--$27$ times above this threshold for $\alpha=0.1$; this is an observation about the experiments, not a consequence of the theorem.
Read as a constraint on the momentum, the condition $\alpha\le1/(1200\rho_2^2\sigma_{\rm DP}^2H^2)$ requires $1-\alpha\ge0.997$ at the CIFAR-10 optimum $B=5000$ and $1-\alpha\ge0.9996$ at $B=1024$, and for language modeling at $B\in\{1024,2048\}$ a momentum window $1/\alpha$ of $2\cdot10^{3}$--$7\cdot10^{3}$ steps, longer than the horizon $K$.
Within our grids it holds only at the largest batch sizes with $1-\alpha=0.99$.
Theorem~\ref{thm:main} also requires the run-in $T=K-\lceil\log2/\alpha\rceil>0$, which for $1-\alpha=0.99$ means $K>70$: this holds for CIFAR-10 at $B=25000$ ($K=160$) but not for language modeling at $B=16384$ ($K=36$).
For the spectral geometry, $H^2\asymp\min\{m,n'\}^2\max\{m,n'\}$ exceeds $d=mn'$ by the factor $\min\{m,n'\}$, so the condition is more restrictive still: $\mathbb E\|Z\|_{\rm nuc}/\sqrt{mn'}>2.5$ for every matrix with $\min\{m,n'\}\ge10$, which covers the weight matrices of our networks, whereas the $\ell_2$ condition holds at CIFAR-10, $B=25000$, $1-\alpha=0.99$ only with a margin of about $1.25$, so no configuration satisfies the spectral condition.
Within our grids, Theorem~\ref{thm:main} therefore covers, in the $\ell_2$ specialization, only CIFAR-10 at $B=25000$ with $1-\alpha=0.99$, and none of the configurations at the observed optima.

\subsection{Spread between optimizers}
\label{app:scaling_evidence}

The following observations are read from the annotated heatmaps.
At the same batch size, and hence the same number of steps, privacy reduces the differences between optimizers on language modeling (Table~\ref{tab:spread}): the spread of the best losses across SGD, Adam, and Muon drops from $1.69$ to $0.09$ nats at $B=128$ and from $2.84$ to $0.22$ nats at $B=1024$. Under privacy the spread grows with the batch size, from $0.08$--$0.09$ nats for $B\le256$ ($\mathrm{NSR}\gtrsim9$) to $0.89$ nats at $B=16384$ ($\mathrm{NSR}=0.2$), and to $1.23$ nats when Scion is included.
At the smallest batch sizes the models also learn little (losses of $6.7$--$6.9$ nats).
On CIFAR-10 the spread grows from at most $1.5$ accuracy points for $B\le512$ to $7.6$ points at $B=25000$ ($\mathrm{NSR}=0.23$); note that Figures~\ref{fig:cnn_baseline_heatmaps}--\ref{fig:cnn_dp_heatmaps} report accuracy, whereas Section~\ref{sec:experiments} reports loss.

\begin{table}[h]
\centering
\small
\caption{Best loss per optimizer and batch size for language modeling (from the heatmaps; lower is better) and spread across optimizers. Non-private runs use $B\le1024$; Scion was run only with privacy and is excluded from the spread for comparability.}
\label{tab:spread}
\begin{tabular}{r ccc c ccc c c}
\toprule
& \multicolumn{4}{c}{non-private} & \multicolumn{5}{c}{private ($\varepsilon=8$)} \\
\cmidrule(lr){2-5}\cmidrule(lr){6-10}
$B$ & SGD & Adam & Muon & spread & SGD & Adam & Muon & spread & Scion \\
\midrule
128 & 5.55 & 3.99 & 3.86 & 1.69 & 6.84 & 6.87 & 6.78 & 0.09 & 6.84 \\
256 & 5.98 & 4.13 & 3.91 & 2.07 & 6.75 & 6.75 & 6.67 & 0.08 & 6.79 \\
512 & 6.45 & 4.43 & 3.97 & 2.48 & 6.62 & 6.49 & 6.55 & 0.13 & 6.67 \\
1024 & 6.94 & 4.82 & 4.10 & 2.84 & 6.55 & 6.33 & 6.52 & 0.22 & 6.54 \\
2048 & -- & -- & -- & -- & 6.68 & 6.36 & 6.51 & 0.32 & 6.52 \\
4096 & -- & -- & -- & -- & 7.04 & 6.56 & 6.70 & 0.48 & 6.51 \\
8192 & -- & -- & -- & -- & 7.52 & 6.77 & 7.04 & 0.75 & 6.61 \\
16384 & -- & -- & -- & -- & 8.02 & 7.13 & 7.46 & 0.89 & 6.79 \\
\bottomrule
\end{tabular}
\end{table}

\section{Detailed Comparison with Existing Guarantees}
\label{app:comparison}

\paragraph{Per-example versus aggregate clipping.}
We compare only with analyses of mechanisms that clip \emph{individual} examples before aggregation, as in~\eqref{eq:per-example-app},
\begin{equation}
\label{eq:per-example-app}
  \widehat g_k \;=\; \frac1B \sum_{i\in S_k} \clip_c\!\big(\nabla f_i(x_k)\big) + z_k .
\end{equation}
Clipping a single vector only rescales it, $\clip_c(g)=a_c(g)\,g$ with $a_c(g)>0$, so a scale-invariant LMO is unaffected by clipping an exact gradient, a batch average, or a client update.
In contrast, the average in~\eqref{eq:per-example-app} need not even be positively correlated with the unclipped mean: with $c=1$ and per-example gradients $(4,-1,-1)$ in one dimension, the unclipped mean is $2/3$ while the clipped mean is $-1/3$, and a larger batch does not help because all examples are already used.
The two mechanisms also differ in privacy calibration: under add/remove adjacency (the Poisson case of Proposition~\ref{prop:privacy}), the sum of individually clipped gradients has $\ell_2$-sensitivity $c$, whereas clipping an aggregate of unbounded gradients only gives sensitivity $2c$ for the aggregate.

\paragraph{Aligning conventions.}
Three conventions differ across the literature.
First, the stationarity measure: a bound $\mathcal O(K^{-1/2})+\mathcal N$ on an averaged \emph{squared} norm gives, by Jensen's inequality, $\mathcal O(K^{-1/4})+\mathcal O(\sqrt{\mathcal N})$ on the norm, and best-iterate, averaged, and expected-minimum statements are ordered as $\mathbb E\min_k\|g_k\|\le\min_k\mathbb E\|g_k\|\le\frac1K\sum_k\mathbb E\|g_k\|$.
Second, the noise convention: writing $s$ for the per-coordinate standard deviation of the noise \emph{after averaging}, our $z_k\sim\mathcal N(0,c^2\sigma_{\rm DP}^2I)$ has $s=c\sigma_{\rm DP}$, a noise multiplier $\sigma_{\rm mult}$ on the sum gives $s=c\sigma_{\rm mult}/B$, and \citet{koloskova2023revisiting} use $s=\sigma_{\rm DP}/\sqrt d$ without the factor $c$.
Third, whether a rate in $K$ holds at a \emph{fixed noise scale}, so that the total privacy budget grows with $K$, or after the noise has been \emph{calibrated} to a fixed $(\varepsilon,\delta)$; a calibrated bound typically does not vanish with $K$.
Theorem~\ref{thm:main} and Appendix~\ref{sec:simplified_analysis} are fixed-noise statements, and Corollary~\ref{cor:calibrated} is the calibrated one.

\paragraph{Centered variance versus raw moments.}
For every $x$,
\begin{equation}
\label{eq:raw-vs-centered}
  \frac1n\sum_{i=1}^n\|\nabla f_i(x)\|_2^2
  \;=\;\|\nabla f(x)\|_2^2+\frac1n\sum_{i=1}^n\|\nabla f_i(x)-\nabla f(x)\|_2^2 .
\end{equation}
Hence a bound on the raw second moment of per-example gradients along the trajectory also bounds the mean gradient, and a light-tail condition such as $\frac1n\sum_i\exp(\|\nabla f_i(x)\|_2^2/G_0^2)\le e$ implies it by Jensen's inequality.
Assumption~\ref{assum:stochastic_oracle} bounds only the centered term on the right, and Lemma~\ref{lem:ray_main} uses nothing else.

\paragraph{Notes on Table~\ref{tab:comparison}.}
\begin{itemize}[leftmargin=1.2em,itemsep=1pt]
\item \citet{koloskova2023revisiting}: Appendix~C.4 (the main-text theorem is stated for the minimum over iterates); the DP extension uses $(L_0,L_1)$-smoothness with a trajectory-dependent constant, and we compare the case $L_1=0$. Their lower bounds show that the floor $\min\{\sigma,\sigma^2/c\}$ cannot be improved for clipped SGD under bounded variance.
\item \citet{yang2022normalized}: Theorem~3.4 (fixed) and Corollary~3.5 (calibrated), under $(L_0,L_1)$-smoothness and $\|\nabla f_i-\nabla f\|\le\tau_0+\tau_1\|\nabla f\|$ almost surely with $c>2\tau_0/(1-\tau_1)$; the bound is on $\mathbb E\min_k\|\nabla f(x_k)\|$.
\item \citet{chen2020understanding}: Theorem~5. The per-example noise is arbitrary; the bound is stated for a weighted clipped stationarity measure plus a Wasserstein distance between the noise distribution and a symmetric reference, which vanishes for symmetric noise.
\item \citet{shulgin2024convergence}: Theorem~2, with $1<q\le2$ and an exact norm-quantile oracle.
\item \citet{vatankhah2025differentially}: Theorem~4.4 and Corollary~4.5 of arXiv v2 (fixed noise) and their Appendix~H (calibrated via advanced composition, without subsampling), with $1<p\le2$, $B=1$, and local smoothness; the bound holds with high probability and carries $\log(K/\eta)$ factors.
\item \citet{kim2026dpmuon} (v2): Theorem~4.3 contains a clipping residual of order $\mathbb E(\|\nabla f_i\|-c)_+$. Corollary~F.2 assumes an exponential moment of per-example gradient norms and uses $c_K\propto\sqrt{\log K}$, giving $K^{-1/4}\sqrt{\log K}$; Corollary~F.3 assumes a raw second moment and uses $c_K\propto K^{1/6}$, giving $K^{-1/6}$. Both are in expectation, for the sum of blockwise nuclear norms of matrix parameters, with a fixed noise \emph{multiplier}, so that the noise standard deviation grows with $c_K$. The vanishing rates require the number of Newton--Schulz steps to grow logarithmically with the horizon.
\item This work: Theorem~\ref{thm:main}, Corollaries~\ref{cor:rates} and~\ref{cor:calibrated}; exact LMO, the conditions~\eqref{eq:main_conditions}, high probability over the privacy noise and expectation over sampling.
\end{itemize}

\paragraph{Discussion.}
Among guarantees for per-example clipping, ours is, to our knowledge, the only one for momentum-LMO dynamics in a general norm that requires only a centered variance bound.
Its fixed-noise rate matches the $\mathcal O(K^{-1/4})$ rate of clipped SGD on the average gradient norm \citep{koloskova2023revisiting}, which is optimal under bounded variance even without clipping \citep{arjevani2023lower}, up to the $(\log K)^{1/4}$ factor from our high-probability treatment of the privacy noise.
Compared with \citet{kim2026dpmuon}, the main differences are as follows.
(i) Our clipping floor $\rho_1\min\{\sigma,\sigma^2/c\}$ vanishes as the sampling noise vanishes, even when $\|\nabla f\|_2\gg c$, because the deterministic part of the clipping bias is a rescaling that the LMO ignores; their residual $\mathbb E(\|\nabla f_i\|-c)_+$ does not vanish in this case.
(ii) We need no moment bounds on the raw gradients, which by~\eqref{eq:raw-vs-centered} would also bound the mean gradient, and no clipping threshold growing with $K$.
(iii) Our bound holds for any norm and with high probability over the privacy noise.
On the other hand, their vanishing rates have no floor (at the price of a growing threshold), and they analyze finite-step Newton--Schulz and a bias-corrected variant, whereas we assume an exact LMO; their rates and ours are therefore not directly comparable.
\citet{shanmugam2026spectral} give a Wedin-type condition under which orthogonalization recovers the leading singular directions of a privatized gradient and document the resulting phase transition empirically; their results concern spectral recovery at a single step rather than convergence of the algorithm.

\section{Extended Related Work}
\label{app:related}

\paragraph{Muon and LMO-based optimization.}
Muon applies an approximate polar-factor transformation to a momentum matrix and can be interpreted as stochastic spectral descent \citep{carlson2015stochastic, kovalev2018stochastic}; practical implementations use a few Newton--Schulz iterations \citep{bernstein2025deriving, shulgin2025beyond}, whose effect on convergence is analyzed by \citet{kim2026muonns}.
The LMO formulation of \citet{pethick2025training} places Muon within a broader class of norm-constrained non-Euclidean methods that includes normalized, sign-based, and coordinate-type updates, extended layer-wise by Gluon \citep{riabinin2025gluon}; trust-region analyses \citep{kovalev2025understanding, kovalev2025non} study the interplay between momentum and noise.
For normalized SGD, \citet{cutkosky2020momentum} established the $K^{-1/4}$ rate with momentum, and \citet{cutkosky2021high} combined clipping, momentum, and normalization in general normed spaces under heavy-tailed noise, with an additive clipping bias controlled by a growing threshold.
Practical and distributed variants include BP-Muon and Drop-Muon \citep{khaled2025muonbp, gruntkowska2025drop} and the distributed or compressed methods Dion, MuLoCo, EF21-Muon, and Compressed Gluon \citep{ahn2025dion, therien2025muloco, gruntkowska2025error, qian2026communication}.

\paragraph{Clipping, normalization, and private optimization.}
Clipping and normalization are used to stabilize training \citep{pascanu2013, merity2017, brown2020, brock2021} and to limit the influence of unreliable updates in robust distributed optimization \citep{karimireddy2021learning, malinovsky2024byzantine, islamov2026byzantine}.
In differential privacy, clipping bounds the sensitivity of each example \citep{Abadi_2016}, also in federated settings \citep{geyer2017differentially, mcmahan2018learning, triastcyn2019federated}, and makes the gradient estimator biased whenever it is active \citep{chen2020understanding, koloskova2023revisiting}.
Per-sample normalization \citep{bu2023automatic, yang2022normalized} removes the dependence on the threshold, and \citet{tang2024dpadambc} show that the nonlinear post-processing of noisy gradients by Adam introduces a bias that can be corrected, the analogue of DP-MuonBC \citep{kim2026dpmuon}.
\citet{compagnoni2026adaptive} compare DP-SignSGD, the $\ell_\infty$ instance of an LMO method, with DP-SGD through stochastic differential equations and find that adaptive methods are preferable at high privacy.
Many convergence analyses of private methods simplify clipping through bounded gradients \citep{zhang2020private, li2022soteriafl, zhang2022understanding, wang2023efficient, lowy2023private, murata2023diff2, wang2024efficient}, bounded noise or heterogeneity \citep{liu2022communication, noble2022differentially}, or inactive clipping \citep{noble2022differentially, zhang2024private}; \citet{das2021convergence} study normalized client updates for convex federated objectives.
Outside privacy, clipping is analyzed for nonsmooth or fast-growing convex functions \citep{shor2012minimization, ermoliev1988stochastic, alber1998projected}, under generalized smoothness \citep{zhang2019gradient, gorbunov2024methods, vankov2024optimizing, lobanov2024linear, hubler2024parameter, tovmasyan2025revisiting}, and under heavy-tailed noise \citep{gorbunov2020stochastic, nguyen2023improved}.
For private nonconvex optimization with Lipschitz losses, faster rates are obtained with variance reduction or tree aggregation \citep{tran2022momentum, arora2023faster, lowy2024make}, and private Frank--Wolfe-type methods in $\ell_p$ geometries are studied by \citet{bassily2021non}.

\section{Additional Experimental Details}
\label{app:experiments}

Our implementation uses PyTorch \citep{paszke2019pytorch} and Opacus \citep{yousefpour2021opacus}, builds on the code of \citet{jordan_cifar10_airbench} for CIFAR-10, and all models are trained on 4 NVIDIA A100-PCIE-40GB GPU. CIFAR-10 is distributed by its authors for research use; FineWeb is released under the ODC-By~1.0 license; the CIFAR Airbench, and NanoGPT code are MIT-licensed.
In all heatmaps, each cell shows the best value over the remaining hyperparameter; the red box marks the best configuration overall and black boxes mark the best value along each column. Due to computational constraints across the extensive hyperparameter sweeps, each cell in the heatmaps represents a single training run. For CIFAR-10, the line plots in Figures~\ref{fig:cnn_baseline_loss_vs_bs}, \ref{fig:epsilon_ablation} report the mean and standard deviation across three independent runs.

\subsection{CIFAR-10}
\label{app:cnn_details}

\textbf{Non-private.} We evaluate learning rates $\gamma \in \{10^{-4}, 10^{-3}, 10^{-2}, 10^{-1}, 1\}$, momentum values $1-\alpha \in \{0, 0.6, 0.8, 0.9, 0.95, 0.99\}$, and batch sizes $B \in \{32, 64, 128, 256, 512, 1024, 2500\}$.
Figure~\ref{fig:cnn_baseline_heatmaps} shows that Muon is robust to the learning rate (accuracy within about one point for $\gamma\in\{10^{-2},10^{-1},1\}$ at most batch sizes), whereas Adam and SGD are more sensitive to it. For SGD, the best learning rate lies at the upper end of the grid ($\gamma=1$) for most batch sizes.

\textbf{Private.} We evaluate learning rates $\gamma \in \{10^{-3}, 10^{-2}, 10^{-1}, 1, 10\}$, momentum values $1-\alpha \in \{0,\allowbreak 0.6,\allowbreak 0.8,\allowbreak 0.9,\allowbreak 0.95,\allowbreak 0.99\}$, and batch sizes $B \in \{64, 128, 256, 512, 1024, 5000, 10000, 25000\}$.
Figure~\ref{fig:cnn_dp_heatmaps} shows that DP-Muon remains comparatively robust to the learning rate, although the window of good learning rates is narrower than without privacy, and that DP-Adam has the same optimal learning rate ($10^{-2}$) at every batch size. For DP-SGD, the best learning rate lies at the upper end of the grid ($\gamma=10$) for $B\ge256$.

\begin{figure}[h]
    \centering
    \includegraphics[width=0.95\textwidth]{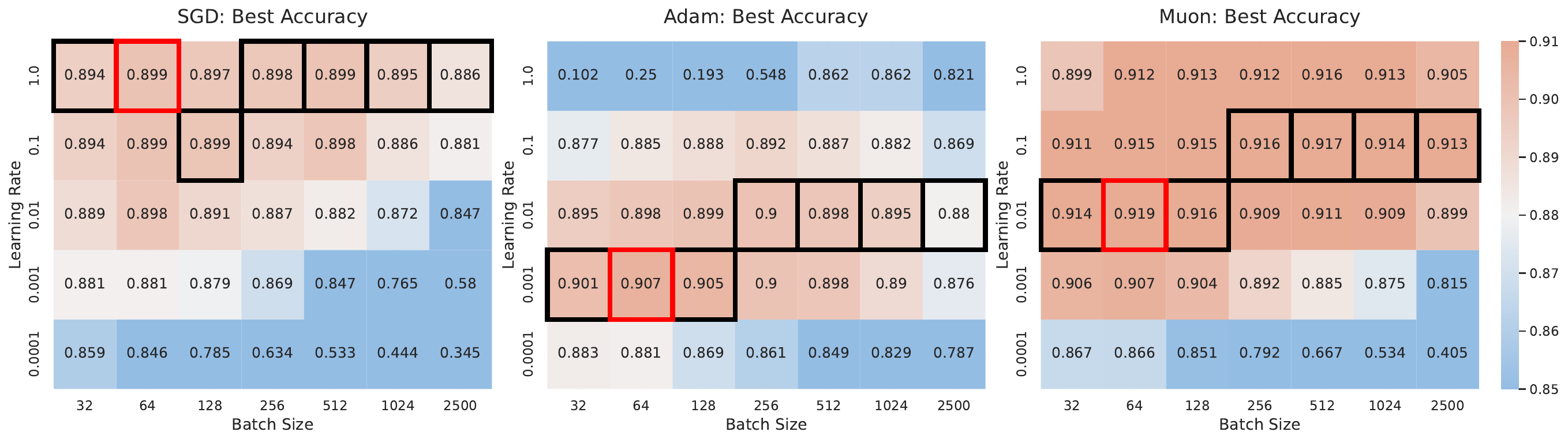}
    \includegraphics[width=0.95\textwidth]{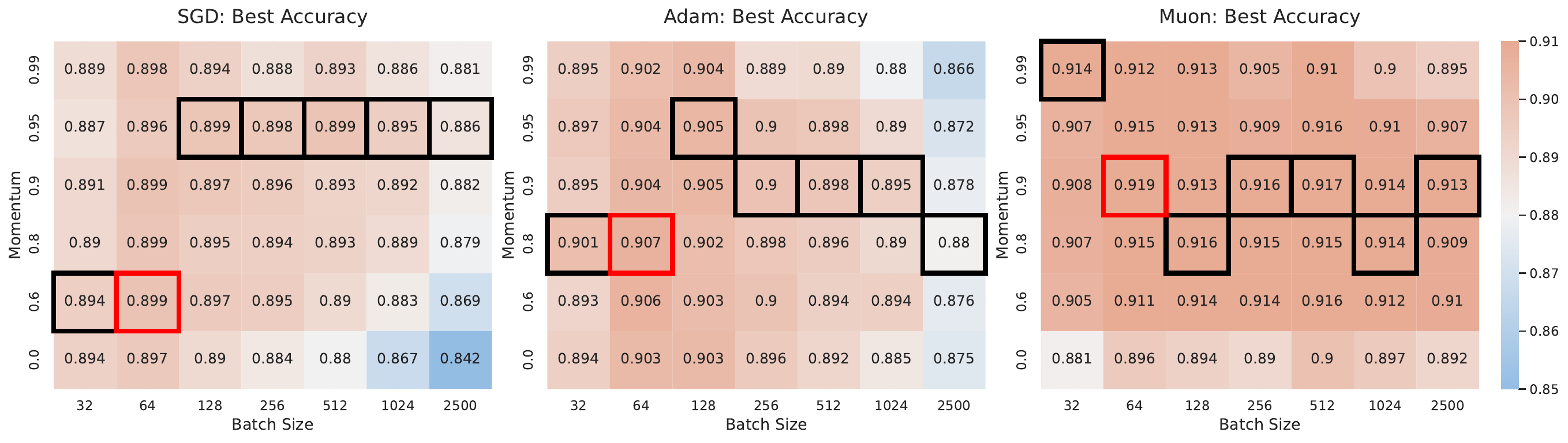}
    \caption{CNN, non-private: peak \textbf{test accuracy} for each batch size and learning rate (\textbf{top}) or momentum (\textbf{bottom}).}
    \label{fig:cnn_baseline_heatmaps}
\end{figure}

\begin{figure}[h]
    \centering
    \includegraphics[width=0.95\textwidth]{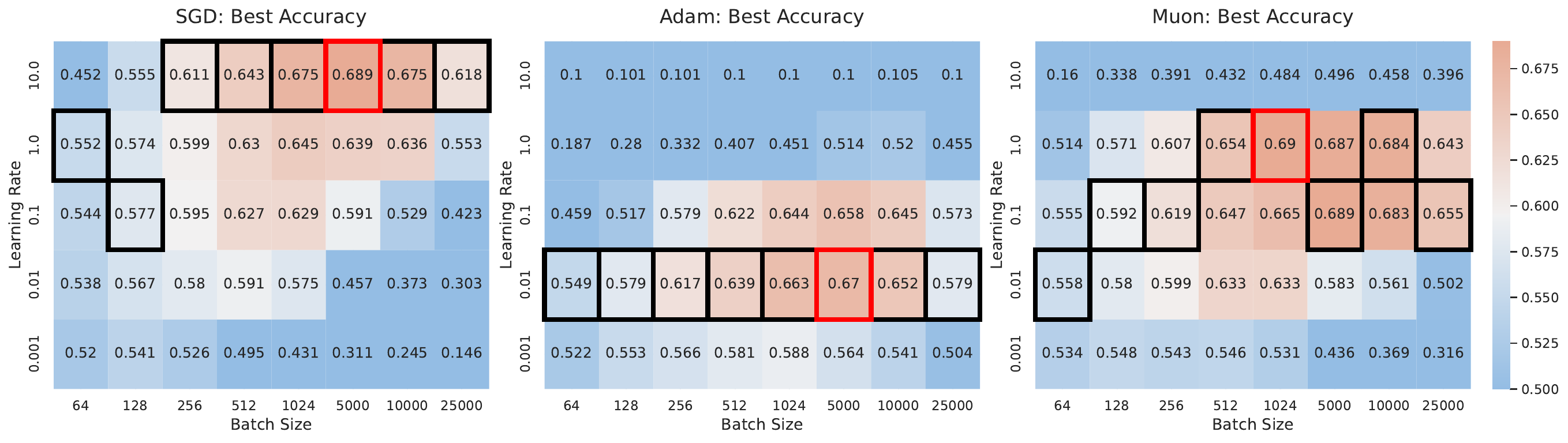}
    \includegraphics[width=0.95\textwidth]{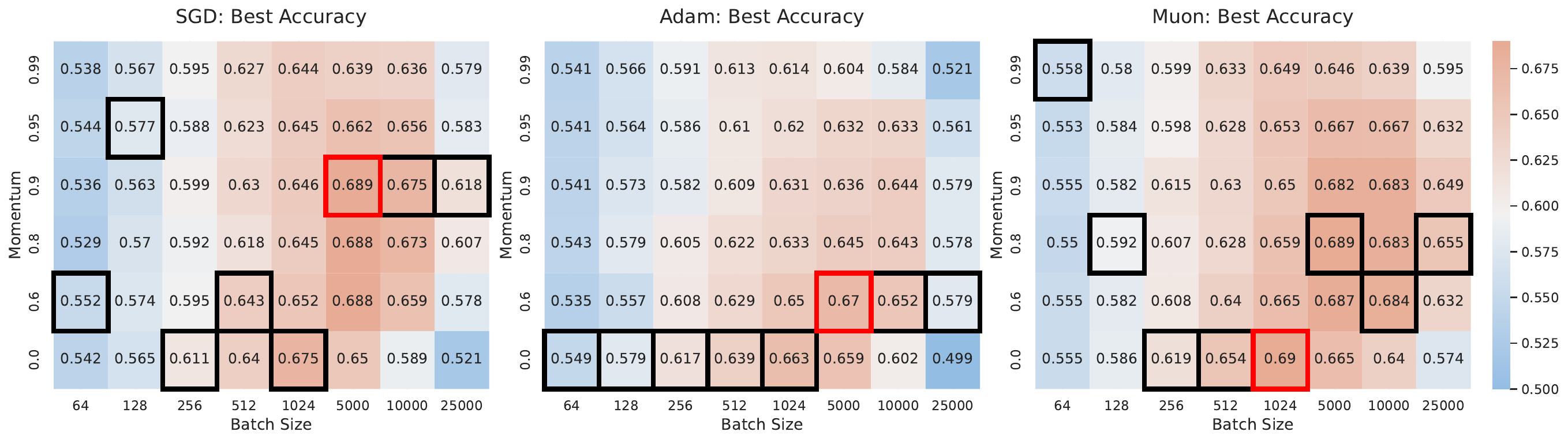}
    \caption{CNN, DP ($\varepsilon=8$): peak \textbf{test accuracy} for each batch size and learning rate (\textbf{top}) or momentum (\textbf{bottom}).}
    \label{fig:cnn_dp_heatmaps}
\end{figure}

\subsection{Privacy-budget ablation}
\label{app:epsilon_ablation}

We vary the privacy budget $\varepsilon\in\{1,3,5,8\}$ at batch size $B=1024$. For each method, we select the best configuration at $\varepsilon=8$ and reuse it for all other values of $\varepsilon$.
As shown in Figure~\ref{fig:epsilon_ablation}, all methods perform similarly. DP-Muon is on par with or slightly better than the baselines for $\varepsilon\in\{1,5,8\}$, with larger variance at $\varepsilon=3$, and DP-Adam is the least accurate at $\varepsilon=1$ and $\varepsilon=8$.

\begin{figure}[h]
    \centering
    \includegraphics[width=0.95\linewidth]{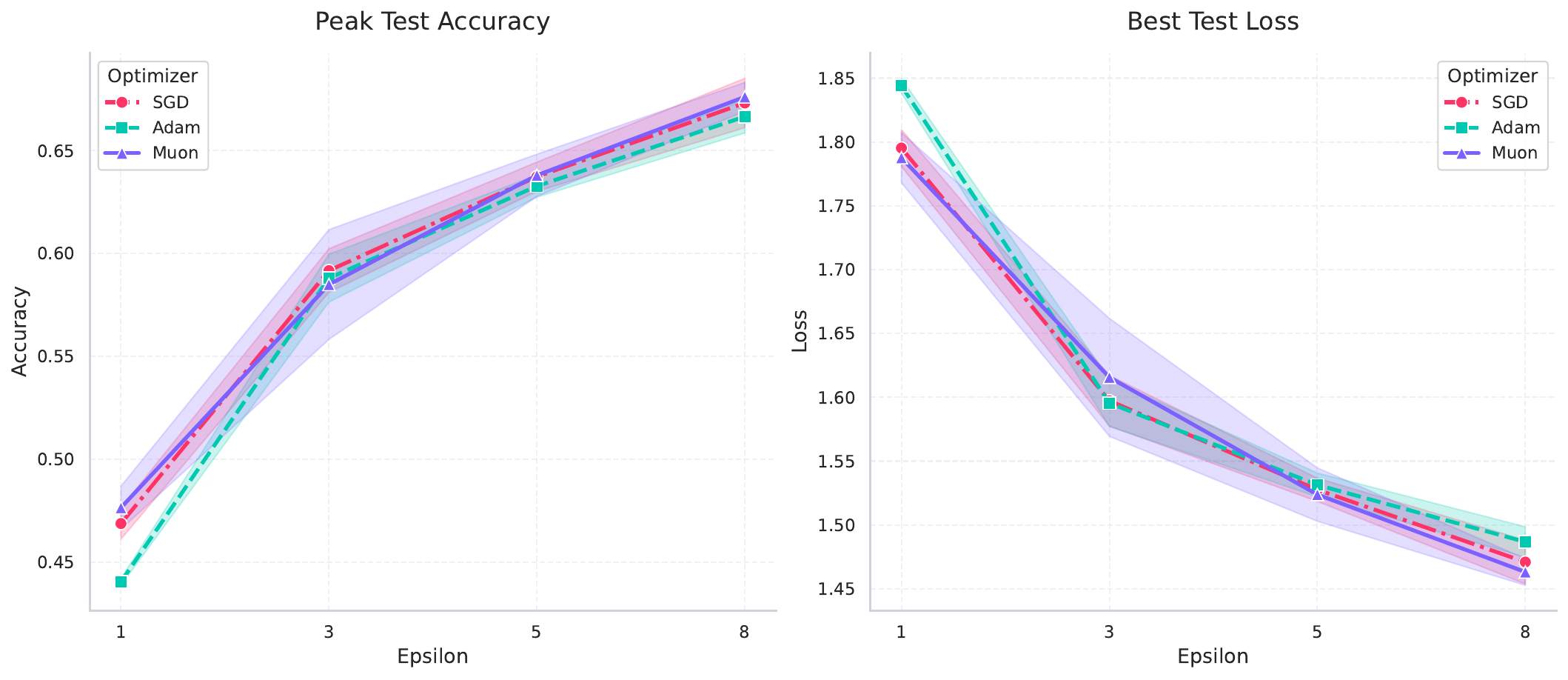}
    \caption{CNN, $B=1024$: peak \textbf{test accuracy} and best \textbf{test loss} versus the privacy budget $\varepsilon$, with hyperparameters tuned at $\varepsilon=8$. Shaded regions show the standard deviation over three runs.}
    \label{fig:epsilon_ablation}
\end{figure}

\subsection{Language modeling}
\label{app:llm_details}

\textbf{Non-private.} We evaluate momentum values $1-\alpha \in \{0.8, 0.9, 0.95, 0.99\}$, batch sizes $B \in \{32,\allowbreak 64,\allowbreak 128,\allowbreak 256,\allowbreak 512,\allowbreak 1024\}$, and the optimizer-specific learning-rate grids shown in Figure~\ref{fig:babygpt_baseline_heatmaps}: $\{10^{-3},5\cdot10^{-3},10^{-2},5\cdot10^{-2},10^{-1}\}$ for SGD and Muon, and $\{5\cdot10^{-4},10^{-3},5\cdot10^{-3},10^{-2},5\cdot10^{-2}\}$ for Adam.
For every optimizer, the optimal learning rate is the same at all batch sizes ($0.05$ for SGD, $10^{-3}$ for Adam, and $0.1$ for Muon, at the upper end of its grid), and Muon achieves the lowest loss at every batch size, with a gap to Adam that grows with $B$.

\textbf{Private.} We evaluate momentum values $1-\alpha \in \{0.8, 0.9, 0.95, 0.99\}$, batch sizes $B \in \{128,\allowbreak 256,\allowbreak 512,\allowbreak 1024,\allowbreak 2048,\allowbreak 4096,\allowbreak 8192,\allowbreak 16384\}$, and the learning rates $\{10^{-3},5\cdot10^{-3},10^{-2},5\cdot10^{-2},10^{-1},0.5,1\}$ for SGD, $\{5\cdot10^{-4},10^{-3},3\cdot10^{-3},5\cdot10^{-3},10^{-2},5\cdot10^{-2}\}$ for Adam, $\{5\cdot10^{-3},10^{-2},5\cdot10^{-2},0.1,0.2,0.5\}$ for Muon, and $\{10^{-2},5\cdot10^{-2},0.1,0.2,0.5,1\}$ for Scion, obtained by refining a common initial grid around each optimizer's best configurations.
Scion applies Newton--Schulz orthogonalization to all hidden layers and sign updates to the embedding and output layers.
Figures~\ref{fig:babygpt_dp_heatmaps}--\ref{fig:babygpt_dp_all_batch_ablation2} show the following.
(i) The optimal learning rate of DP-Muon increases with the batch size ($0.05$ for $B\le512$, $0.1$ for $1024\le B\le4096$, $0.2$ for $B\ge8192$), while that of DP-Adam stays in $[10^{-3},3\cdot10^{-3}]$.
(ii) At its optimal learning rate, the loss of DP-Muon varies by about $0.1$ nats or less over $1-\alpha\in\{0.8,0.9,0.95\}$ at every batch size, and $1-\alpha=0.99$ is worse at small batch sizes; the selected momentum therefore changes from one batch size to the next without a clear trend.
(iii) For DP-SGD, the loss depends strongly on the momentum at a given learning rate, reflecting the coupling between the two through the effective step size.
(iv) For Scion at $B=16384$, the best learning rate is the largest one in the grid.

\begin{figure}[h]
    \centering
    \includegraphics[width=0.95\textwidth]{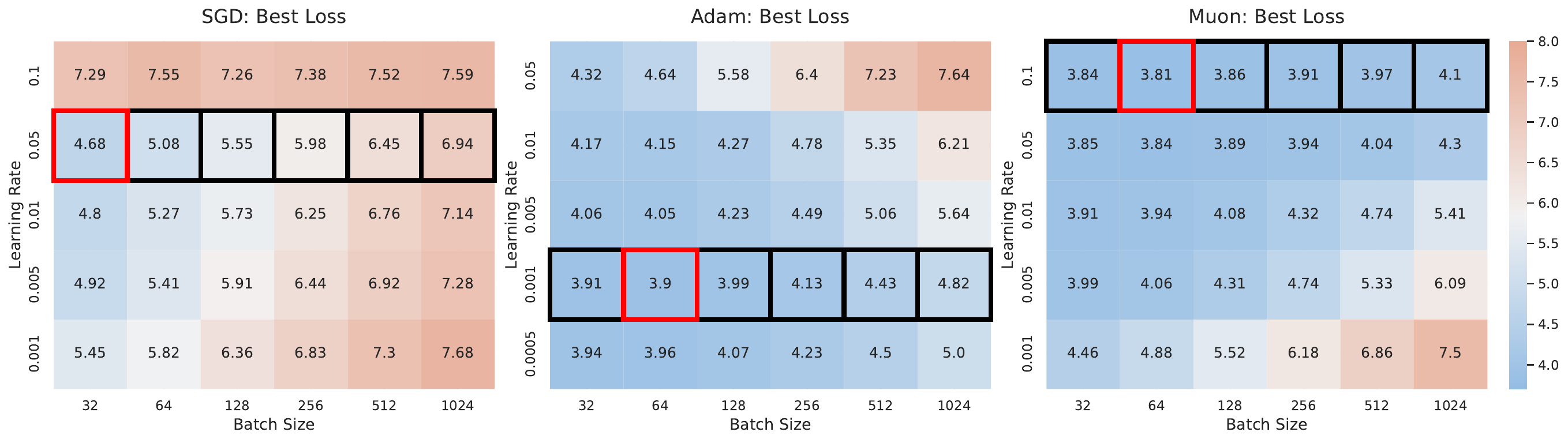}
    \includegraphics[width=0.95\textwidth]{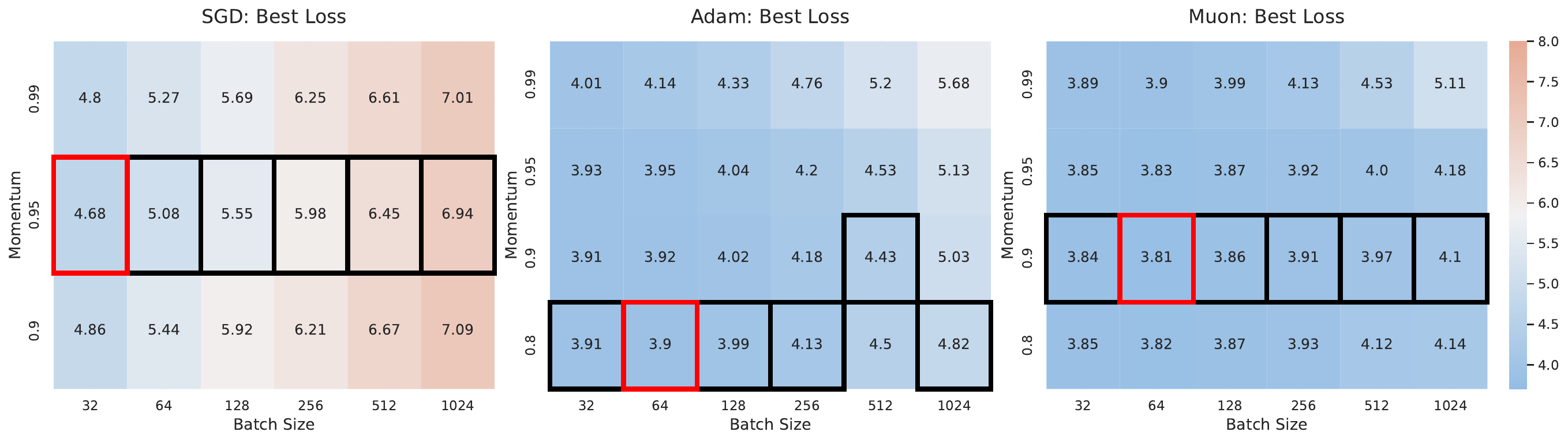}
    \caption{GPT (30M), non-private: best loss for each batch size and learning rate (\textbf{top}) or momentum (\textbf{bottom}).}
    \label{fig:babygpt_baseline_heatmaps}
\end{figure}

\begin{figure}[h]
    \centering
    \includegraphics[width=0.95\textwidth]{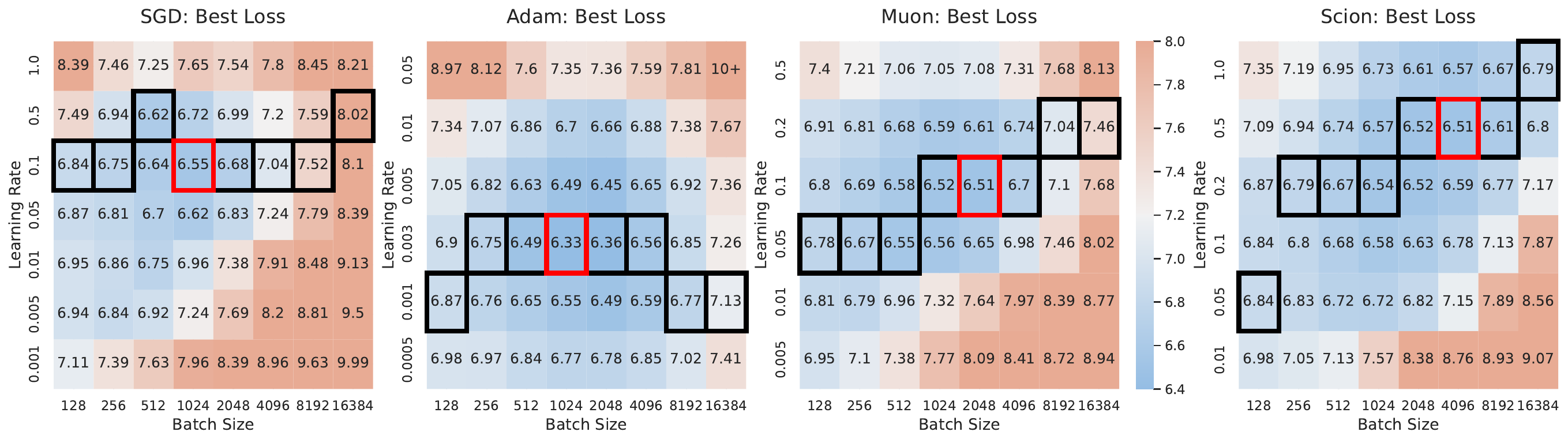}
    \includegraphics[width=0.95\textwidth]{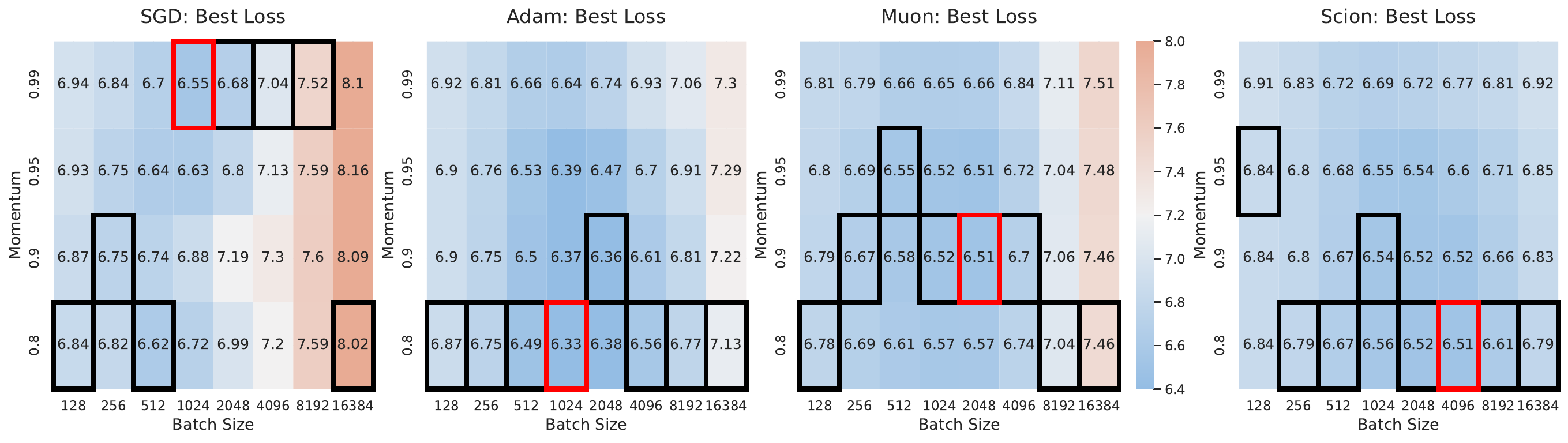}
    \caption{GPT (30M), DP ($\varepsilon=8$): best loss for each batch size and learning rate (\textbf{top}) or momentum (\textbf{bottom}).}
    \label{fig:babygpt_dp_heatmaps}
\end{figure}

\begin{figure}[h]
    \centering
    \includegraphics[width=0.95\textwidth]{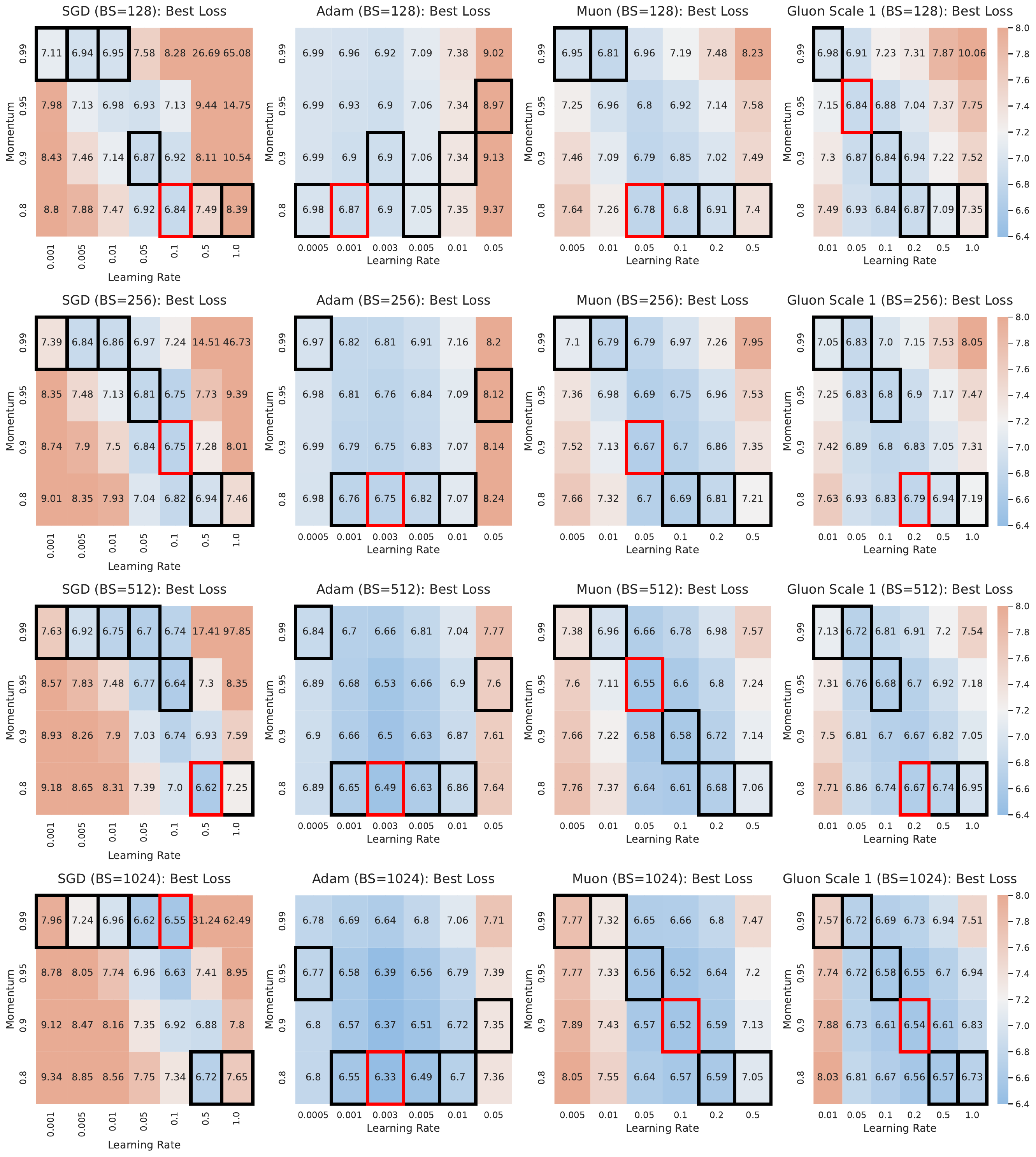}
    \caption{GPT (30M), DP ($\varepsilon=8$): best loss for every learning rate (columns) and momentum (rows), for $B\in\{128,256,512,1024\}$. ``Gluon Scale 1'' denotes the Scion variant used in Figure~\ref{fig:babygpt_baseline_loss_vs_bs}.}
    \label{fig:babygpt_dp_all_batch_ablation1}
\end{figure}

\begin{figure}[h]
    \centering
    \includegraphics[width=0.95\textwidth]{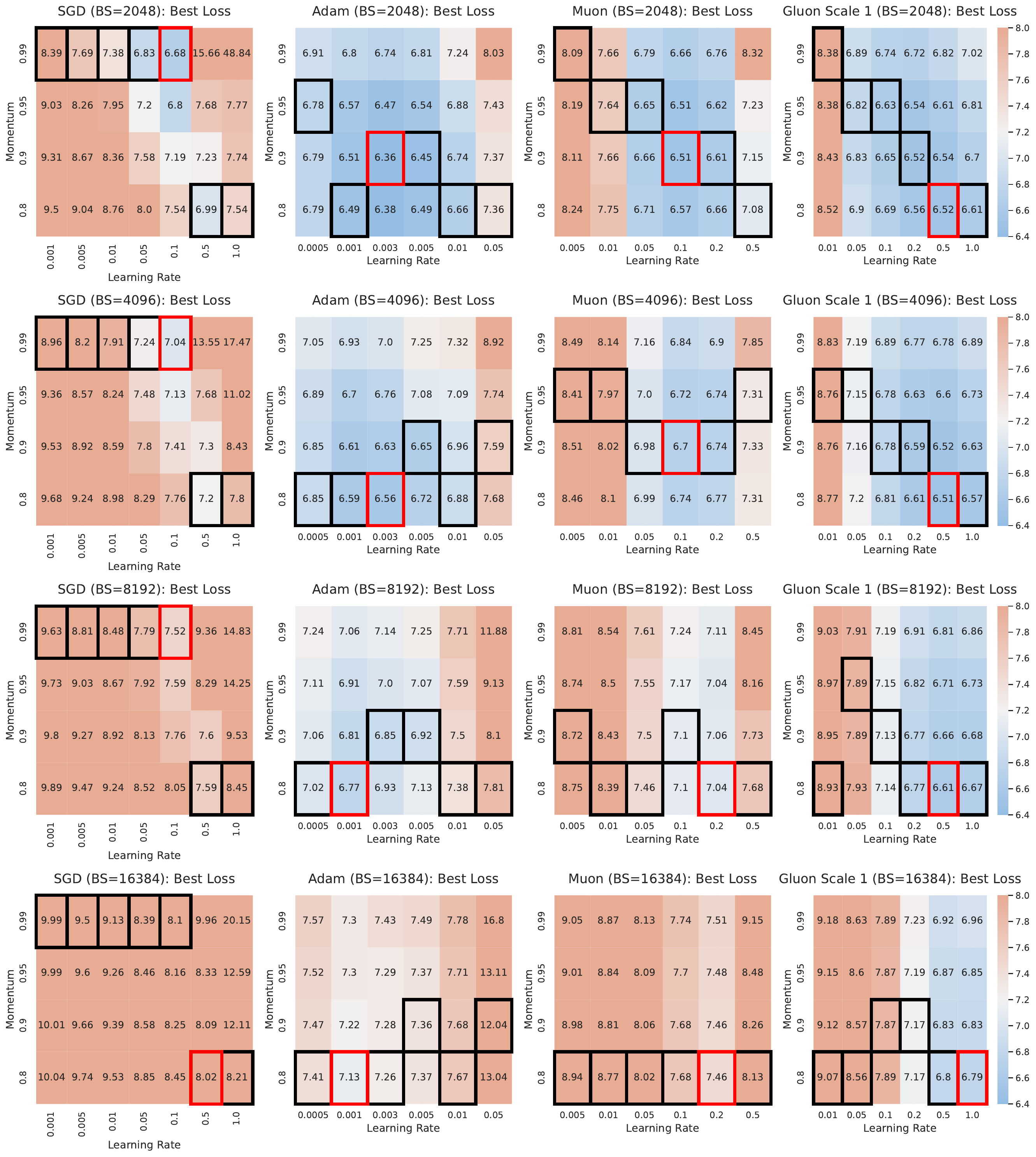}
    \caption{GPT (30M), DP ($\varepsilon=8$): best loss for every learning rate (columns) and momentum (rows), for $B\in\{2048,4096,8192,16384\}$. ``Gluon Scale 1'' denotes the Scion variant used in Figure~\ref{fig:babygpt_baseline_loss_vs_bs}.}
    \label{fig:babygpt_dp_all_batch_ablation2}
\end{figure}
\end{document}